\pdfoutput=1
\documentclass[11pt]{article}

\usepackage[T1]{fontenc}
\usepackage[utf8]{inputenc}
\usepackage[margin=1in]{geometry}
\usepackage{amsmath,amssymb,amsthm}
\usepackage{mathtools}
\usepackage{float}
\usepackage{flafter}
\usepackage{capt-of}
\usepackage{booktabs}
\usepackage{needspace}
\usepackage{algorithm}
\usepackage{algpseudocode}
\usepackage{xcolor}
\usepackage{tikz}
\usetikzlibrary{positioning,calc,arrows.meta,decorations.pathreplacing,fit,backgrounds}
\usepackage{graphicx}
\usepackage{microtype}
\usepackage[round,authoryear]{natbib}
\usepackage{xurl}
\usepackage[bookmarks=true]{hyperref}

\hypersetup{
  colorlinks=true,
  linkcolor=blue,
  citecolor=blue,
  urlcolor=blue,
  pdftitle={Memory--Batch Trade-offs in Lipschitz Bandits},
  pdfauthor={Zicheng Lyu and Zengfeng Huang},
  pdfsubject={Learning theory for memory-constrained and batched Lipschitz bandits},
  pdfkeywords={Lipschitz bandits, batched bandits, finite memory, minimax regret, information routing}
}
\providecommand{\doi}[1]{\href{https://doi.org/#1}{doi:\nolinkurl{#1}}}
\theoremstyle{plain}
\newtheorem{theorem}{Theorem}[section]

\newtheorem{corollary}[theorem]{Corollary}
\newtheorem{lemma}[theorem]{Lemma}
\theoremstyle{definition}
\newtheorem{definition}[theorem]{Definition}
\theoremstyle{remark}

\newcommand{\X}{\mathcal X}
\newcommand{\E}{\mathbb E}
\newcommand{\PP}{\mathbb P}
\newcommand{\Dclass}{\mathfrak D}
\newcommand{\Pbar}{\overline{\mathbb P}}
\newcommand{\Ebar}{\overline{\mathbb E}}
\newcommand{\Reg}{\mathrm{Reg}}
\newcommand{\Lip}{\mathrm{Lip}}
\newcommand{\RecMap}{\mathsf{Rec}}
\newcommand{\Tr}{\mathsf{Tr}}
\newcommand{\LCB}{\operatorname{LCB}}
\newcommand{\UCB}{\operatorname{UCB}}
\newcommand{\proofparagraph}[1]{\par\medskip\noindent\textbf{#1.}\enspace}
\algrenewcommand\algorithmicrequire{\textbf{Input:}}
\algrenewcommand\algorithmicensure{\textbf{Output:}}

\title{Memory--Batch Trade-offs in Lipschitz Bandits}
\author{%
  Zicheng Lyu\\
  Fudan University\\
  Shanghai, China\\
  \texttt{lyuzicheng@gmail.com}
  \and
  Zengfeng Huang\thanks{Corresponding author.}\\
  Fudan University and Shanghai Innovation Institute\\
  Shanghai, China\\
  \texttt{huangzf@fudan.edu.cn}%
}
\date{}

\begin{document}
\maketitle
\begin{abstract}
Lipschitz bandits admit near-optimal regret
$\widetilde O_d(T^{(d+1)/(d+2)})$
with little memory under full adaptivity,
or with few batches under unrestricted memory.
We characterize the minimax expected pseudo-regret over $T$
rounds with $W$ bits of memory and at most $B$ batches in
dimension $d$.
For every memory budget $W$, we prove a lower bound of order
$T^{\frac{d+2}{d+3}}
(1+(B-1)W)^{-\frac{1}{d(d+3)}}$.
When $W\gtrsim_d\log(eT)$, algorithms with fixed batch boundaries
attain, up to logarithmic factors, the larger
of this bound and the optimal $B$-batch regret with unrestricted
memory.
The analysis separates fine-scale comparisons from the regional
information needed to allocate their samples at low regret.
Attaining $\widetilde O_d(T^{(d+1)/(d+2)})$ regret requires both
$B=\Omega_d(\log\log T)$ and
$(B-1)W=\widetilde{\Omega}_d(T^{d/(d+2)})$.
With one bit of memory, minimax regret is
$\widetilde\Theta_d(TB^{-1/(d+2)})$ under adaptive
batch boundaries while $\Theta_d(T)$ under fixed boundaries.
\end{abstract}

\section{Introduction}\label{sec:intro}

Memory-limited bandits restrict how much information from
past actions and rewards a learner can retain
\citep{cover1968note,liau2018stochastic};
batched bandits restrict when this information can affect
its actions \citep{perchet2016batched,gao2019batched}.

We study these two constraints jointly in stochastic Lipschitz
bandits on $[0,1]^d$ over $T$ rounds
\citep{agrawal1995continuum,kleinberg2004nearly,auer2007improved,
bubeck2011xarmed}.
Near-optimal regret $\widetilde O_d(T^{(d+1)/(d+2)})$
is achievable with $O_d(\log(eT))$ bits of memory under
full adaptivity \citep{zhu2025lipschitz}, or with
$O_d(\log\log T)$ batches under unrestricted memory
\citep{feng2022lipschitz}.
We ask:

\begin{center}
\emph{How does minimax regret depend jointly on memory and
batch constraints in Lipschitz bandits?}
\end{center}

In the jointly constrained model, the learner uses at most $B$
batches and maintains at most $W$ bits of reward-dependent mutable
state after every observation.
Within a batch, rewards may update the state but cannot alter
that batch's prescribed action sequence or boundary.

For $K$-armed bandits, \citet{huang2026few} show that, under suitable
conditions on $W$ and $T$, $\widetilde\Theta(K/W)$ batches are
necessary and sufficient for $\widetilde O(\sqrt{KT})$ regret
with $W$ bits of memory.
For Lipschitz bandits, directly applying this result to a grid
makes the number of actions depend on the desired approximation
accuracy.
For some memory and batch budgets, our two-scale construction
achieves a polynomially smaller regret upper bound than the
finite-arm guarantee optimized over regular grid sizes
(Section~\ref{app:finite-arm-comparison}).
\begin{theorem}[Informal main theorem; see
Theorem~\ref{thm:joint-frontier}]
\label{thm:informal-frontier}
Fix an integer $d\ge1$. For sufficiently large $T$,
$1\le B\le T$, and $W\gtrsim_d\log(eT)$, the minimax
expected pseudo-regret in the model above satisfies
\[
\mathfrak R_T(B,W)
=
\widetilde{\Theta}_d\!\left(
T^{\frac{d+1}{d+2}}
\vee
\frac{T^{\beta_{d,B}}}{B^2}
\vee
T^{\frac{d+2}{d+3}}
\bigl(1+(B-1)W\bigr)^{-\frac1{d(d+3)}}
\right),
\]
where
$\beta_{d,B}:=(d+1)/\{(d+2)[1-(d+2)^{-B}]\}$.
The lower bound holds up to a $d$-dependent constant for all
$W\ge0$, even with adaptively chosen batch boundaries;
the upper bound uses fixed boundaries.
\end{theorem}
The first two terms characterize the unrestricted-memory
$B$-batch regret up to logarithmic factors.
The third captures the additional memory constraint and,
for fixed $B\ge2$, decreases with $W$.
Section~\ref{subsec:formal-frontier} identifies the resulting
regret regimes.
\paragraph{Contributions.}
\begin{enumerate}
\item[(i)] \textbf{Minimax memory--batch tradeoff.}
We characterize the minimax expected pseudo-regret up to
logarithmic factors for $W\gtrsim_d\log(eT)$.
In particular, logarithmic memory requires
$\widetilde{\Omega}_d(T^{d/(d+2)})$ batches for near-optimal
regret, polynomially more than the $O_d(\log\log T)$ batches
sufficient with unrestricted memory.

\item[(ii)] \textbf{One-bit regret and adaptive batch boundaries.}
Below logarithmic memory, the three-term maximum need not
be tight, and fixed boundaries can incur a polynomial loss.
For $W=1$, minimax regret is
$\widetilde\Theta_d(TB^{-1/(d+2)})$ with adaptive boundaries,
but $\Theta_d(T)$ with fixed boundaries, for every $1\le B\le T$
(Section~\ref{sec:timing-sensitivity}).
\end{enumerate}

\subsection{Technique overview}
\label{subsec:technique-overview}

Applying \citet[Theorem~2]{huang2026few} to a regular grid
of resolution $r$ uses $K\asymp_d r^{-d}$ actions and gives
an approximation term $O_d(Tr)$.
Under its assumptions and suppressing logarithmic factors,
the resource condition $K\lesssim BW$ requires
$r\gtrsim_d(BW)^{-1/d}$.
Thus, this guarantee ties finer approximation to a larger
memory--batch budget.

We use coarse cells of diameter $O_d(s)$ to guide comparisons
at resolution $r\le s$.
By Lipschitz continuity, if a cell's representative has gap
$O_d(s)$, every point in the cell also has gap $O_d(s)$.
Fine comparisons therefore cost $O_d(s)$ regret per pull
(Figure~\ref{fig:two-scale-overview}).

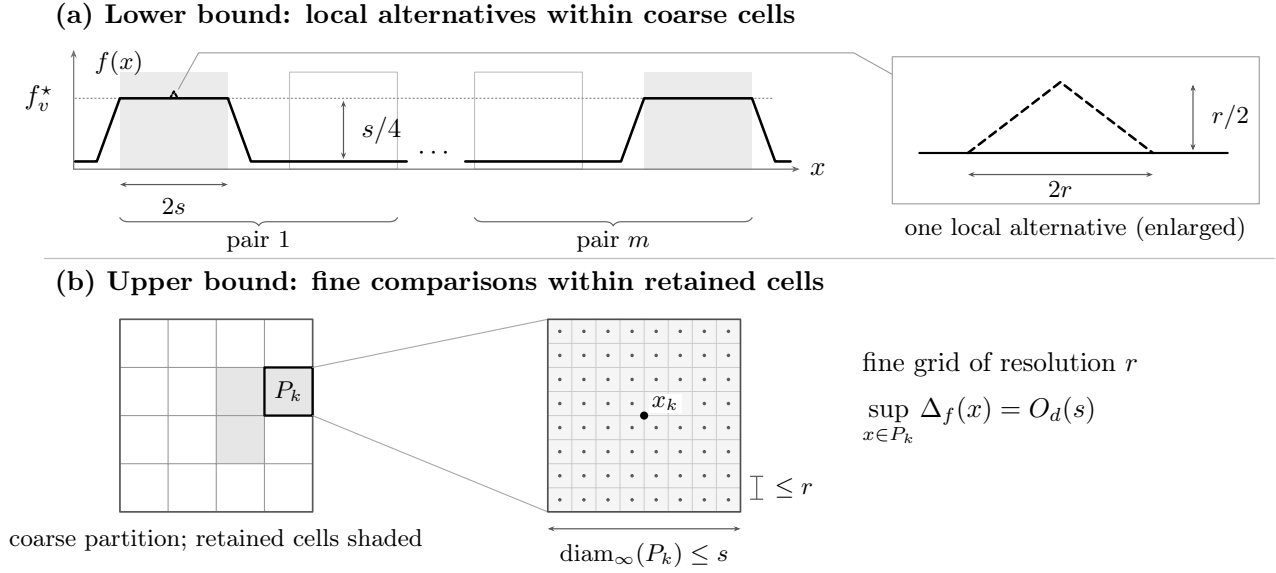
\begin{figure}[htbp]
\centering
\resizebox{\textwidth}{!}{%
\begin{tikzpicture}[
  x=1cm,y=1cm,color=black,
  font=\fontsize{10}{11.5}\selectfont,
  line cap=round,line join=round,
  heading/.style={font=\fontsize{10}{12}\selectfont\bfseries,anchor=west},
  note/.style={font=\fontsize{9}{10.4}\selectfont},
  axis/.style={draw=black!55,line width=.45pt,-{Stealth[length=3pt]}},
  dim/.style={draw=black!65,line width=.45pt,{Stealth[length=2.5pt]}-{Stealth[length=2.5pt]}}
]
\selectcolormodel{gray}
\path[use as bounding box] (0,0) rectangle (16.2,7.35);

\node[heading] at (.12,7.08) {(a) Lower bound: local alternatives within coarse cells};
\foreach \a/\b in {1.10/2.50,7.90/9.30}
  \fill[black!8] (\a,5.10) rectangle (\b,6.36);
\foreach \a/\b in {3.30/4.70,5.70/7.10}
  \draw[black!25,line width=.4pt] (\a,5.10) rectangle (\b,6.36);
\draw[axis] (.50,5.10) -- (9.92,5.10) node[right] {$x$};
\draw[axis] (.50,5.10) -- (.50,6.65);
\node[note,anchor=west] at (.64,6.48) {$f(x)$};
\draw[black!45,densely dotted,line width=.45pt] (.52,6.02) -- (9.58,6.02);
\node[anchor=east] at (.37,6.02) {$f_v^\star$};
\draw[black,line width=.95pt]
  (.52,5.20) -- (.80,5.20) -- (1.10,6.02)
  -- (2.50,6.02) -- (2.80,5.20) -- (4.82,5.20);
\draw[black,line width=.95pt]
  (5.58,5.20) -- (7.60,5.20) -- (7.90,6.02)
  -- (9.30,6.02) -- (9.60,5.20) -- (9.80,5.20);
\node at (5.20,5.31) {$\cdots$};
\draw[black,dash pattern=on 1.2pt off .8pt,line width=.8pt]
  (1.75,6.02) -- (1.80,6.11) -- (1.85,6.02);
\draw[dim] (4.00,5.23) -- (4.00,5.99);
\node[anchor=west] at (4.10,5.61) {$s/4$};
\draw[dim] (1.10,4.90) -- (2.50,4.90);
\node[note,anchor=north] at (1.80,4.85) {$2s$};
\draw[black!60,decorate,decoration={brace,mirror,amplitude=3pt},line width=.45pt]
  (1.10,4.47) -- (4.70,4.47);
\draw[black!60,decorate,decoration={brace,mirror,amplitude=3pt},line width=.45pt]
  (5.70,4.47) -- (9.30,4.47);
\node[note,anchor=north] at (2.90,4.43) {pair $1$};
\node[note,anchor=north] at (7.50,4.43) {pair $m$};

\draw[black!40,line width=.4pt]
  (1.86,6.14) -- (2.12,6.61) -- (10.62,6.61) -- (11.12,6.33);
\draw[black!35,line width=.45pt] (11.12,4.65) rectangle (15.88,6.56);
\draw[black,line width=.8pt] (11.47,5.31) -- (15.47,5.31);
\draw[black,densely dashed,line width=.95pt]
  (12.10,5.31) -- (13.30,6.23) -- (14.50,5.31);
\draw[dim] (15.02,5.34) -- (15.02,6.20);
\node[note,anchor=west] at (15.13,5.78) {$r/2$};
\draw[dim] (12.10,5.03) -- (14.50,5.03);
\node[note,anchor=north] at (13.30,5.08) {$2r$};
\node[note] at (13.50,4.32) {one local alternative (enlarged)};

\draw[black!25,line width=.45pt] (.12,3.94) -- (16.08,3.94);
\node[heading] at (.12,3.62) {(b) Upper bound: fine comparisons within retained cells};
\begin{scope}[shift={(1.10,.65)}]
  \foreach \i/\j in {2/1,2/2,3/2}
    \fill[black!10] (\i*.625,\j*.625) rectangle +(.625,.625);
  \draw[step=.625,black!40,line width=.4pt] (0,0) grid (2.5,2.5);
  \draw[black!65,line width=.6pt] (0,0) rectangle (2.5,2.5);
  \draw[black,line width=.85pt] (1.875,1.25) rectangle +(.625,.625);
  \node[note] at (2.1875,1.5625) {$P_k$};
\end{scope}
\node[note] at (2.35,.28) {coarse partition; retained cells shaded};

\draw[black!35,line width=.4pt] (3.60,2.525) -- (6.65,3.15);
\draw[black!35,line width=.4pt] (3.60,1.90) -- (6.65,.65);
\begin{scope}[shift={(6.65,.65)}]
  \fill[black!4] (0,0) rectangle (2.5,2.5);
  \draw[step=.3125,black!20,line width=.25pt] (0,0) grid (2.5,2.5);
  \foreach \i in {0,...,7}
    \foreach \j in {0,...,7}
      \fill[black!65] ({(\i+.5)*.3125},{(\j+.5)*.3125}) circle (.65pt);
  \draw[black!70,line width=.6pt] (0,0) rectangle (2.5,2.5);
  \fill[black] (1.25,1.25) circle (1.5pt);
  \node[note,anchor=south west,fill=white,inner sep=1pt] at (1.31,1.28) {$x_k$};
\end{scope}
\draw[dim] (6.65,.43) -- (9.15,.43);
\node[note,anchor=north] at (7.90,.38) {$\operatorname{diam}_\infty(P_k)\le s$};
\draw[black!65,line width=.45pt,|-|] (9.36,.80625) -- (9.36,1.11875);
\node[note,anchor=west] at (9.46,.9625) {$\le r$};
\node[anchor=west] at (10.60,2.57) {fine grid of resolution $r$};
\node[anchor=west] at (10.60,1.83)
  {$\displaystyle\sup_{x\in P_k}\Delta_f(x)=O_d(s)$};
\end{tikzpicture}%
}
\caption{Two-scale geometry.
\textbf{(a)} One cell in each of $m\asymp_d s^{-d}$ pairs is optimal
under the base instance; the other has gap $s/4$.
A local alternative adds a bump of height $r/2$ and radius $r\le s/16$
within an optimal cell.
\textbf{(b)} On the confidence event, retained cells collectively contain
a maximizer and have gap $O_d(s)$ throughout; their fine grids have
resolution $r$.}
\label{fig:two-scale-overview}
\end{figure}

\paragraph{Lower bound: allocating local tests is informative.}
At admissible scales $r\le s/16$, the hard family has
$\asymp_d s^{-d}$ pairs of cells, each with one optimal cell
and one with gap $\Theta(s)$.
Independent uniform bits select the optimal cells, which contain
$\asymp_d r^{-d}$ disjoint radius-$r$ neighborhoods.
Each neighborhood supports a local improvement of height
$\Theta(r)$.
Although one optimal cell suffices on a base instance,
regret below $c_dTr$ on every local alternative forces
$\Omega(r^{-2})$ visits to each selected neighborhood with
constant probability under the base law.
The resulting $\Omega_d(r^{-d-2})$ sample requirement yields
the sequential lower bound when it exceeds $T$.

Under a base instance, testing the unselected cells equally
extensively would incur $\Omega_d(sr^{-d-2})$ regret.
If expected regret on every base instance is also below
$c_dsr^{-d-2}$, visit counts predict the regional choices with
a constant average advantage over chance, revealing
$\Omega_d(s^{-d})$ bits.
Yet the boundary states determine the actions and carry at most
$(B-1)W$ bits conditional on public randomness \citep{huang2026few}.
Thus, at admissible scales, $s^{-d}\ge C_d(B-1)W$ gives the
lower bound $\Omega_d(\min\{Tr,sr^{-d-2}\})$
(Section~\ref{sec:lower-main}).
The batch-depth term follows from the unrestricted-memory
lower bound of \citet[Theorem~8]{feng2022lipschitz}, adapted
to our reward model.

\paragraph{Upper bound: reuse memory across regions.}
For $W\gtrsim_d\log(eT)$, the first construction processes
$O_d(s^{-d})$ coarse cells in small groups \citep{chaudhuri2020regret}.
It reconstructs each group's \emph{active set}, the cells not yet
eliminated, through comparisons with a retained reference
\citep{liau2018stochastic}.
On a confidence event, surviving cells contain a maximizer
and have gap $O_d(s)$ throughout.
We compare fine actions within these cells and retain the
best estimated action across groups.
For feasible scales, fine comparisons and exploitation incur
$\widetilde O_d(sr^{-d-2}+Tr)$ expected regret;
active-set reconstruction fits within the same bound.
Processing the groups uses
$\widetilde O_d(1+s^{-d}/W)$ batches in total.
Balancing $Tr$ and $sr^{-d-2}$ yields the joint term
in both bounds.
A second construction retains active sets across spatial levels
to use fewer batches.
Both use fixed boundaries and the same final refinement;
Section~\ref{sec:upper-main} bounds their construction costs.

\paragraph{One-bit memory.}
For one bit of memory, we additionally bound how well the
state retained at a batch boundary can distinguish local
alternatives, uniformly of the batch length.
The matching policy chooses batch lengths so that comparison
outcomes can be recovered from later round and batch indices
(Section~\ref{sec:timing-sensitivity}).
\section{Related work}\label{sec:related}

\paragraph{Lipschitz bandits and limited adaptivity.}
Classical continuum-armed and metric bandit results relate
regret to the geometry of the action space and its near-optimal
regions
\citep{agrawal1995continuum,kleinberg2004nearly,auer2007improved,
kleinberg2008multi,bubeck2011xarmed,magureanu2014lipschitz}.
Related partition-based methods appear in tree search and
deterministic optimistic optimization
\citep{coquelin2007bandit,munos2011optimistic},
and adaptive discretization extends to adversarial Lipschitz
bandits \citep{podimata2021adaptive}.
Batched feedback has been studied for finite arms, linear bandits,
and other structured models
\citep{perchet2016batched,gao2019batched,jin2021almost,
esfandiari2021regret,ruan2021linear,hanna2023efficient,
sawarni2024generalized}, alongside broader restrictions on
interaction in limited-round learning
\citep{agarwal2017learning}.
\citet{feng2022lipschitz} study batched Lipschitz bandits
without memory constraints, with regret bounds depending
on zooming dimension.
Our joint characterization concerns the full Lipschitz class
on $[0,1]^d$.

\paragraph{Memory-constrained and streaming bandits.}
Memory bounds use different units and access rules:
state configurations and bits
\citep{cover1968note,cover1970two,zhu2025lipschitz},
stored words \citep{liau2018stochastic},
or retained arms
\citep{chaudhuri2020regret,maiti2021multi,assadi2020exploration,
agarwal2022sharp,wang2023tight,assadi2024best}.
In multi-pass streaming, an evicted arm cannot be sampled again
until it reappears in a later pass, although sampling can adapt within a pass
\citep{agarwal2022sharp}.
Consequently, pass bounds concern restrictions on revisiting
arms and do not directly yield bounds on batched action updates.
Communication constraints differ as well:
\citet{lau2026batched} restrict each batch's returned message
to one bit but allow decisions to use the history of received
bits, rather than imposing a fixed mutable-state budget.
The sequential Lipschitz algorithm of \citet{zhu2025lipschitz}
revisits regions through iterative deepening, deciding whether
to explore a cell's descendants after sampling it.
Here that decision can affect actions only in a later batch.
Other related formulations study memory-constrained adversarial bandits \citep{xu2021memory}, bounded-space online learning
\citep{srinivas2022memory,peng2023onlineprediction,peng2023near},
and streaming tradeoffs
\citep{jin2021optimalstreaming,li2023tight}.

\paragraph{Joint memory and adaptivity constraints.}
We build on the finite-arm analysis of \citet{huang2026few},
whose hard family already separates a base gap from a smaller
local improvement.
Our lower bound combines their boundary-state argument with
localized change-of-measure comparisons based on classical
likelihood-ratio tools
\citep{mannor2004sample,kaufmann2016complexity}.
Our upper bound draws on repeated comparisons with a retained
reference \citep{liau2018stochastic} and processing candidates
in small groups \citep{chaudhuri2020regret}.
Related blockwise comparisons under joint memory and batch
constraints appear in \citet[Section~5]{huang2026few}.
Our construction keeps the reference fixed across an entire
pass and uses the final selections to guide finer comparisons
within Lipschitz cells.
\section{Model and minimax regret}\label{sec:prelim}

\subsection{Batched policies with limited memory}
\label{subsec:policy-model}

Throughout, $d,T,B,W$ are integers satisfying
$d,T\ge1$, $1\le B\le T$, and $W\ge0$.

\paragraph{Lipschitz Bandits.}
Let $\X=[0,1]^d$ with the sup norm, and let $\Lip_1(\X)$
denote the class of one-Lipschitz functions from $\X$ to $[0,1]$.
For $f\in\Lip_1(\X)$, let $\Dclass(f)$ consist of families
of reward distributions $\nu=(\nu_x)_{x\in\X}$ on $[0,1]$,
with mean $f(x)$ and measurable dependence on $x$.
At round $t$, the learner selects an action $A_t\in\X$
and observes a reward $Y_t$ whose conditional distribution,
given the past and $A_t$, is $\nu_{A_t}$.
Define
\[
f^\star:=\sup_{x\in\X}f(x),\qquad
\Delta_f(x):=f^\star-f(x),\qquad
\Reg_T(f):=\sum_{t=1}^T\Delta_f(A_t).
\]

\paragraph{Policies.}
Let $\omega$ be an independent public random seed.
The complete reward-dependent mutable state satisfies
$M_t\in\{0,1\}^W$ for $t=0,\ldots,T$, with
$M_0=\iota(\omega)$.
A policy partitions the $T$ rounds into at most $B$
nonempty batches, with $\tau_0=0$.
Each batch boundary is chosen adaptively at the start
of that batch. The policy is specified by measurable maps
satisfying
\begin{align}
\tau_{b+1}
&=\psi_b(\omega,\tau_b,M_{\tau_b}),
\label{eq:committed-endpoint-map}\\
A_t
&=\phi_{b,t}(\omega,\tau_b,M_{\tau_b}),
\label{eq:committed-action-map}\\
M_t
&=\mathcal U_{b,t}(\omega,M_{t-1},A_t,Y_t),
\qquad \tau_b<t\le\tau_{b+1}.
\notag
\end{align}
The parameters $d,T,B,W$, public random seed $\omega$,
round and batch indices, and the current batch's fixed
boundary and action sequence remain available as $M_t$
changes and do not count toward $W$.
Boundaries fixed in advance depend only on $\omega$.
The model follows \citet{huang2026few};
Appendix~\ref{app:boundary-memory} gives the full specification.

\paragraph{Minimax regret.}
Let $\mathfrak A_{B,W}$ denote the class of policies defined above.
The minimax expected pseudo-regret is
\[
\mathfrak R_T(B,W):=
\inf_{\mathcal A\in\mathfrak A_{B,W}}
\sup_{f\in\Lip_1(\X),\,\nu\in\Dclass(f)}
\E_{\nu,\mathcal A}\bigl[\Reg_T(f)\bigr],
\]
where the expectation is over reward and policy randomness.

\subsection{Minimax regret bounds}\label{subsec:formal-frontier}

Write $\ell_T:=\log(eT)$, $\alpha_d:=(d+1)/(d+2)$,
and $\beta_{d,B}:=(d+1)/\bigl((d+2)[1-(d+2)^{-B}]\bigr)$.
We use $[n]=\{1,\ldots,n\}$, $\vee$ for maximum, and $\wedge$
for minimum. Unsubscripted logarithms are natural.
The relations $\lesssim_d$, $\gtrsim_d$, and $\asymp_d$ hide
constants depending only on $d$; a tilde additionally hides powers
of $\log(eT)$. Unless fixed in a statement, $c_d,C_d>0$ may change
between occurrences. We fix a deterministic action $x_\circ\in\X$
for padding and fallback choices.

\begin{theorem}[Minimax regret under memory and batch constraints]
\label{thm:joint-frontier}
There exist constants $c_d,C_d,T_d>0$, depending only on $d$,
such that, for every $T\ge T_d$ and all $B,W$,
\begin{equation}\label{eq:joint-frontier-lower}
\mathfrak R_T(B,W)
\ge c_d\left[
T^{\frac{d+1}{d+2}}
\vee\frac{T^{\beta_{d,B}}}{B^2}
\vee T^{\frac{d+2}{d+3}}
\bigl(1+(B-1)W\bigr)^{-\frac1{d(d+3)}}
\right].
\end{equation}
If $W\ge C_d\log(eT)$, then
\begin{equation}\label{eq:joint-frontier-upper}
\mathfrak R_T(B,W)
\le C_d\log(eT)\left[
T^{\frac{d+1}{d+2}}
\vee\frac{T^{\beta_{d,B}}}{B^2}
\vee T^{\frac{d+2}{d+3}}
\bigl(1+(B-1)W\bigr)^{-\frac1{d(d+3)}}
\right].
\end{equation}
The lower bound allows batch boundaries chosen adaptively at batch
starts; the upper bound is attained with fixed batch boundaries.
\end{theorem}
For smaller memory the lower bounds remain valid but need not
be tight; Section~\ref{sec:timing-sensitivity} gives the one-bit
rate and its fixed-boundary separation.

Removing the memory constraint gives the policy class
$\mathfrak A_{B,\infty}$ and minimax regret
$\mathfrak R_T^{\mathrm{bat}}(B)$, with
\begin{equation}\label{eq:unrestricted-batch-frontier}
\mathfrak R_T^{\mathrm{bat}}(B)
=\widetilde\Theta_d\!\left(
T^{\alpha_d}\vee\frac{T^{\beta_{d,B}}}{B^2}
\right),\qquad 1\le B\le T.
\end{equation}
The exponent $\beta_{d,B}$ matches that in the lower bound of
\citet{feng2022lipschitz}.
For sufficiently large $C_d$, $B\ge C_d\log\log T$
gives $\mathfrak R_T^{\mathrm{bat}}(B)
=\widetilde\Theta_d(T^{\alpha_d})$.

\paragraph{Dominant terms and resource tradeoffs.}
Assume $W\ge C_d\ell_T$ and write
$Q_{T,B}:=T^{\beta_{d,B}-\alpha_d}/B^2$.
Relative to $T^{\alpha_d}$, the three terms in
\eqref{eq:joint-frontier-lower} are
\begin{equation}\label{eq:normalized-frontier-terms}
1,\qquad Q_{T,B},\qquad
\left(
\frac{T^{d/(d+2)}}{1+(B-1)W}
\right)^{\frac1{d(d+3)}}.
\end{equation}
\begin{center}
\begin{minipage}{\linewidth}
\captionof{table}{Dominance conditions, including ties.
A dash means no additional restriction.}
\label{tab:dominant-terms}
\par\vspace{\baselineskip}
\renewcommand{\arraystretch}{1.1}
\begin{tabular*}{\linewidth}{
  @{\extracolsep{\fill}} l c l @{}
}
\toprule
\textbf{Dominant term}
& \textbf{Batch condition}
& \multicolumn{1}{c}{\textbf{Memory condition}} \\
\midrule
Sequential
& $Q_{T,B}\le1$
& $1+(B-1)W\ge T^{d/(d+2)}$ \\
Batch-depth
& $Q_{T,B}\ge1$
& $1+(B-1)W\ge T^{d/(d+2)}Q_{T,B}^{-d(d+3)}$ \\
Joint
& ---
& $1+(B-1)W\le
   T^{d/(d+2)}(1\vee Q_{T,B})^{-d(d+3)}$ \\
\bottomrule
\end{tabular*}
\end{minipage}
\end{center}

For fixed $B\ge2$, once the joint term is at most
$T^{\alpha_d}(1\vee Q_{T,B})$, the maximum no longer depends
on $W$, and further memory can improve minimax regret
by at most logarithmic factors.
For fixed $W$, increasing $B$ can move from batch-depth
to joint and then sequential dominance, possibly skipping
the joint regime.
Appendix~\ref{app:regime-comparisons} derives the boundaries
and shows why all three terms are needed.

\paragraph{Batch Complexity Corollary.}
For $\varrho>0$, define
$\mathcal B_T(W;\varrho):=\min\{B\in[T]:
\mathfrak R_T(B,W)\le\varrho\ell_TT^{\alpha_d}\}$,
with $\min\varnothing:=+\infty$.

\begin{corollary}[Batch complexity for near-optimal regret]
\label{cor:batch-complexity}
For a dimension-dependent constant $\varrho_d>0$, all sufficiently
large $T$, and $W\gtrsim_d\log(eT)$,
\begin{equation}\label{eq:batch-complexity-tilde}
\mathcal B_T(W;\varrho_d)
=\widetilde\Theta_d\!\left(
\log\log T\vee\frac{T^{d/(d+2)}}{W}
\right).
\end{equation}
\end{corollary}
Appendix~\ref{app:batch-complexity-proof} gives the explicit
logarithmic factors.

\section{Lower bounds}\label{sec:lower-main}

The regional family in Figure~\ref{fig:two-scale-overview}(a)
yields the sequential and joint lower bounds.
A separate batching argument completes
\eqref{eq:joint-frontier-lower}.

\subsection{Two-scale hard family}

Fix $0<s\le s_d$ and $0<r\le s/16$, with $s_d>0$
sufficiently small depending only on $d$.
Choose $2m\asymp_d s^{-d}$ centers $x_{j,a}\in[1/4,3/4]^d$,
indexed by $j\in[m]$ and $a\in\{0,1\}$, separated by at
least $8s$ in the sup norm, and define the cells
$P_{j,a}:=\{x\in\X:\|x-x_{j,a}\|_\infty\le s\}$.
In each cell, choose the same number $q\asymp_d(s/r)^d$
of centers $z_{j,a,k}$ satisfying
$\|z_{j,a,k}-x_{j,a}\|_\infty\le s/4$ and pairwise separation
at least $4r$.
The neighborhoods
$G_{j,a,k}:=\{x\in\X:\|x-z_{j,a,k}\|_\infty\le r\}$
are disjoint and contained in their parent cells.

Write $(u)_+:=\max\{u,0\}$ and
$d_\infty(x,S):=\inf_{y\in S}\|x-y\|_\infty$ for nonempty $S$.

\begin{definition}[Base and local Bernoulli instances]
\label{def:routing-hard-family}
For $v\in\{0,1\}^m$, define the base mean and its local alternatives by
\begin{align}
f_v(x)
&:=\frac14+\max_{j\in[m]}
  \left(\frac s4-\frac12d_\infty(x,P_{j,v_j})\right)_+,
  \notag\\
f^+_{v,j,k}(x)
&:=f_v(x)+\frac12
  \left(r-\|x-z_{j,v_j,k}\|_\infty\right)_+,
\qquad j\in[m],\ k\in[q].
\label{eq:hard-local-mean}
\end{align}
The instances have Bernoulli rewards with these means.
\end{definition}

Under $f_v$, the bit $v_j$ selects an optimal cell in each pair:
$f_v=1/4+s/4$ on $P_{j,v_j}$, whereas $f_v=1/4$ on
$P_{j,1-v_j}$.
The local alternative adds a bump of height $r/2$ within
$G_{j,v_j,k}$, making $z_{j,v_j,k}$ uniquely optimal.
Outside this neighborhood, the reward law is unchanged and the
gap under the alternative is at least $r/2$.
All means are one-Lipschitz and lie in $[1/4,3/4]$.
Appendix~\ref{app:hard-family-geometry} gives the packing
and verifies these properties.
The construction separates the local improvement $r/2$ from
the regret cost $s/4$ of sampling an unselected cell, with
$mq\asymp_d r^{-d}$ selected neighborhoods across
$m\asymp_d s^{-d}$ regional choices.

\subsection{A two-scale regret lower bound}
\label{subsec:regional-information}

\begin{lemma}[Two-scale lower bound]
\label{lem:routing-memory}
There exist dimension-dependent constants $c_d,C_d>0$
such that, for the scales above with $Tr^2\ge C_d$, if
\begin{equation}\label{eq:routing-memory-condition}
(B-1)W\le c_ds^{-d},
\end{equation}
then
\begin{equation}\label{eq:two-scale-lower}
\mathfrak R_T(B,W)
\ge c_d\min\{Tr,sr^{-d-2}\}.
\end{equation}
The bound allows batch boundaries chosen adaptively
at batch starts.
\end{lemma}

\begin{proof}[Proof sketch]
Fix a policy $\mathcal A\in\mathfrak A_{B,W}$, and let
$\mathcal R$ be its largest expected regret over the base
instances and local alternatives.
Suppose $\mathcal R<c_d\min\{Tr,sr^{-d-2}\}$.
Let $V$ be uniform on $\{0,1\}^m$, independently of $\omega$,
and write $\Pbar,\Ebar$ for probability and expectation
under the mixture of base instances $f_V$.

Since $\mathcal R<c_dTr$,
Corollary~\ref{lem:selected-probe-testing} gives
$n\asymp r^{-2}$ such that each selected neighborhood
receives at least $n$ visits with probability at least $5/8$.
Set $\rho_{j,a}:=q^{-1}\sum_{k=1}^q
\mathbf1\{N_{j,a,k}(T)\ge n\}$,
where $N_{j,a,k}(T)$ counts visits to $G_{j,a,k}$,
and mark $P_{j,a}$ when $\rho_{j,a}\ge1/4$.
Then $\Ebar\rho_{j,V_j}\ge5/8$, while each marked
unselected cell incurs at least $sqn/16$ regret. Thus
\[
\Pbar(\rho_{j,V_j}\ge1/4)\ge\frac12,
\qquad
\frac1m\sum_{j=1}^m\Pbar(\rho_{j,1-V_j}\ge1/4)
\le\frac{16\mathcal R}{msqn}\le\frac18,
\]
using $\rho_{j,a}\le1$, $msqn\asymp_d sr^{-d-2}$,
and $\mathcal R<c_dsr^{-d-2}$ with sufficiently small $c_d$.

Choose the uniquely marked side as $\widehat V_j$, with
independent fair tie-breaking otherwise.
Its average error is at most $5/16$.
Conditional data processing and the binary entropy bound give
\begin{equation}\label{eq:regret-information-bridge}
I(V;A_{1:T}\mid\omega)
\ge I(V;\widehat V\mid\omega)
\ge m\bigl[1-h_2(5/16)\bigr]
\ge c'_ds^{-d},
\end{equation}
where $h_2$ is binary entropy in bits and $c'_d>0$
depends only on $d$.
By \eqref{eq:boundary-information}, the boundary-state argument
of \citet{huang2026few} gives
$I(V;A_{1:T}\mid\omega)\le(B-1)W\le c_ds^{-d}$.
Choosing $c_d<c'_d$ gives a contradiction.
Since the policy was arbitrary, \eqref{eq:two-scale-lower} follows.
Appendix~\ref{app:information-proof} gives the full calculation.
\end{proof}

\paragraph{Completing the lower bound.}
Balancing $Tr$ and $sr^{-d-2}$ gives
$r\asymp_d(s/T)^{1/(d+3)}$.
With $s\asymp_d(1+(B-1)W)^{-1/d}$, this yields the
joint term in \eqref{eq:joint-frontier-lower}
whenever these scales are admissible.
Appendix~\ref{app:memory-optimization} completes the
sequential and joint bounds over all memory and batch budgets.

Even with unrestricted memory, the learner can adapt to
rewards only between batches.
The adaptive-grid lower bound of
\citet[Theorem~8]{feng2022lipschitz} quantifies this
restriction, giving regret of order
$T^{\beta_{d,B}}/B^2$.
Lemma~\ref{lem:batch-depth-transfer} transfers this bound
to Bernoulli rewards and policies with at most $B$ batches
for $B\ge2$.
Appendix~\ref{app:lower} completes the proof of
\eqref{eq:joint-frontier-lower}, including $B=1$.
\section{Upper bounds}\label{sec:upper-main}

We prove \eqref{eq:joint-frontier-upper} for $W\ge C_d\ell_T$
using fixed batch boundaries.
We first analyze refinement within a selected set of coarse
cells, then give two ways to construct this set under
different memory and batch budgets.

\subsection{Refinement within active cells}
\label{subsec:active-set-interface}

Fix $0<r\le s\le1$.
Let $\mathcal P_s=\{P_1,\ldots,P_K\}$ be a fixed dyadic partition
of $\X$ into $K\asymp_ds^{-d}$ cells of diameter at most $s$,
with midpoint representatives $x_k$.
Subdividing each cell to side length at most $r$ gives an
$r$-net $\mathcal N_r(P_k)$ of descendant midpoints, with
$O_d((s/r)^d)$ points per cell. Write $\ell_r:=\log(e/r)$.

\begin{definition}[Safe active set]
\label{def:safe-active-set}
Fix a sufficiently large dimension-dependent constant $\kappa_d$.
An active set $\mathcal C\subseteq\mathcal P_s$ is safe for $f$ if
\[
\left(\bigcup_{P\in\mathcal C}P\right)
\cap\operatorname*{arg\,max}_{x\in\X}f(x)\ne\varnothing,
\qquad
\sup_{P\in\mathcal C}\sup_{x\in P}\Delta_f(x)\le\kappa_ds.
\]
We represent $\mathcal C$ by its indicator vector
$Z=(\mathbf1\{P_k\in\mathcal C\})_{k=1}^K$.
\end{definition}

Given a safe active set, reserve a block of
$O_d(r^{-2}\ell_r)$ pulls for each fine-grid action.
In blocks belonging to inactive cells, play a coarse
representative $\bar x$ whose index is available from the
stored state and whose gap satisfies
$\Delta_f(\bar x)\le\kappa_ds$.
This keeps the block schedule fixed in advance.
Process active-action estimates one at a time, retaining only
the largest estimated mean and its action, then exploit that action.
This uses $O_d(\ell_r)$ auxiliary bits beyond $Z$.

Every refinement pull has gap $O_d(s)$, and the retained
cells contain an $O_d(r)$-optimal fine-grid action.
If $C_dr^{-d-2}\ell_r\le T/2$ and the complete schedule
finishes by $T$, refinement and exploitation therefore have
expected regret $O_d(sr^{-d-2}\ell_r)$ and $O_d(Tr)$,
respectively, including estimation failures.
Cells may be selected and refined one group at a time,
provided their union is safe and the recorded action and
estimate are retained across groups.
Write $\Reg_{\mathrm{build}}(f)$ for construction regret and
$\Reg_{\mathrm{ref+exp}}(f)$ for refinement and exploitation regret.
Appendix~\ref{app:refinement} proves the conditional guarantee
used by both constructions below.

\subsection{Active-set construction in groups}
\label{subsec:serialized-main}

Storing one group's indicators at a time gives the following
counterpart to Lemma~\ref{lem:routing-memory}.

\begin{lemma}[Two-scale upper bound]
\label{lem:two-scale-upper}
There exist constants $C_d,T_d>0$, depending only on $d$,
such that, for $T\ge T_d$ and scales $0<r\le s\le1/16$,
the conditions
\begin{equation}\label{eq:two-scale-upper-resources}
W\ge C_d\ell_T,\qquad
B-1\ge C_d\bigl(1+s^{-d}/W\bigr)(1+\log\ell_T)^2,\qquad
C_dr^{-d-2}\ell_T\le T/2,
\end{equation}
imply the existence of a policy
$\mathcal A\in\mathfrak A_{B,W}$ with fixed batch boundaries
satisfying
\begin{equation}\label{eq:serialized-fixed-scale}
\sup_{f\in\Lip_1(\X),\,\nu\in\Dclass(f)}
\E_{\nu,\mathcal A}\bigl[\Reg_T(f)\bigr]
\le C_d\bigl[sr^{-d-2}\ell_T+Tr\bigr].
\end{equation}
\end{lemma}

\paragraph{Construction.}
Partition the coarse cells into groups whose indicator vectors
fit in memory together with $O_d(\ell_T)$ auxiliary bits.
Let $L_{\mathrm{ser}}$ be the number of passes.
Each pass visits every group, initializing its indicators to one
and reconstructing its active set from fresh samples.

A group's best representative need not be near-optimal globally.
The common reference allows comparisons across groups without
storing their indicator vectors together. Choose a fixed initial reference $x_{k_0}$, with $k_0\in[K]$.
In pass $i$, use the common reference action $x_{k_{i-1}}$.
A prescribed prefix of the first batch estimates its mean;
let $\lambda_i$ be the resulting lower confidence bound,
held fixed for the rest of the pass.
For an active candidate $k$, let $\widehat\mu_{k,\ell}$ and
$a_\ell$ denote its estimated mean and confidence radius
at comparison step $\ell=1,\ldots,i$.
Suppressing pass and group indices, update
\begin{equation}\label{eq:serialized-reference-update}
Z_{k,\ell}
=
Z_{k,\ell-1}
\mathbf1\{
\widehat\mu_{k,\ell}+a_\ell\ge\lambda_i-\zeta_i
\},
\qquad
\zeta_i:=s\mathbf1\{i=L_{\mathrm{ser}}\}.
\end{equation}
Inactive entries remain zero, and their prescribed sampling
blocks use the reference action.
Each nonfinal pass retains the surviving candidate with
largest final lower confidence bound as the next reference.
In the final pass, refine each group's selected cells before
discarding its indicator vector, retaining the largest
estimated mean and its action across groups.

\begin{proof}[Proof sketch]
Write $f_K^\star:=\max_{k\le K}f(x_k)$.
On the joint confidence event $\mathcal G$, a best representative
survives each pass, and the reference after each nonfinal pass $i$
satisfies $f_K^\star-f(x_{k_i})\le2a_i$.
Since $a_{L_{\mathrm{ser}}-1},a_{L_{\mathrm{ser}}}=O_d(s)$,
the final reference lower bound satisfies
$f_K^\star-C_ds\le\lambda_{L_{\mathrm{ser}}}\le f_K^\star$.
With the final threshold lowered by $s$, the union $\mathcal C$
of selected cells satisfies
\[
\{P_k:f_K^\star-f(x_k)\le s\}
\subseteq\mathcal C
\subseteq\{P_k:f_K^\star-f(x_k)\le C_ds\}.
\]
The first inclusion preserves a maximizer's cell.
The second, together with $f^\star-f_K^\star\le s$ and
Lipschitz continuity, bounds the gap throughout each retained
cell by $O_d(s)$.

Larger sample sizes are used only for candidates that survive
earlier comparisons.
The schedule in Appendix~\ref{app:serialized} gives
$\Reg_{\mathrm{build}}(f)\le C_ds^{-d-1}\ell_T$ on $\mathcal G$.
Since $r\le s$, this is bounded by the fine-comparison term
in \eqref{eq:serialized-fixed-scale}.
Combining \eqref{eq:refinement-good-event-contract} with
$\PP(\mathcal G^\complement)\le r$ and $\ell_r\le C_d\ell_T$
proves the regret bound.
Let $J_s=O_d(1+s^{-d}/W)$ be the number of groups.
With $L_{\mathrm{ser}}=O(1+\log\ell_T)$ passes, the total
number of batches is
$J_s\bigl(L_{\mathrm{ser}}(L_{\mathrm{ser}}+1)/2+1\bigr)+1$:
$i$ comparison batches per group in pass $i$, one final
refinement batch per group, and one exploitation batch.
This gives \eqref{eq:two-scale-upper-resources}; the full proof
is in Appendix~\ref{app:serialized}.
\end{proof}

\subsection{Hierarchical active-set construction}
\label{subsec:in-memory-main}

Reconstructing active sets saves memory but costs batches.
With fewer batches, we retain the active set and refine
its cells one spatial level per batch.
For $B\ge3$ and dyadic scales $0<r\le s\le1/16$, let
$L=L_B(s)$ and $b_{B,r}\asymp_d\log(eB/r)$ be as specified
in Appendix~\ref{app:in-memory}.
Use $L$ spatial levels with dyadic scales
$1=u_0\ge u_1\ge\cdots\ge u_L=s$.

Initialize $\mathcal C_0=\{\X\}$.
At level $\ell$, let $\mathcal Q_\ell\subseteq\mathcal P_{u_\ell}$
consist of the cells contained in some
$P\in\mathcal C_{\ell-1}$.
Use two fixed sweeps over the level-$\ell$ cells, reserving
$O_d(b_{B,r}u_\ell^{-2})$ pulls per cell in each sweep.
For cells in $\mathcal Q_\ell$, sample their midpoints;
other blocks use a representative of the preceding active set.

Let $\widehat\mu_{P,1}$ and $\widehat\mu_{P,2}$ be the rounded
means from the two sweeps, with confidence radius
$a_\ell=u_\ell/32$.
For nonempty $\mathcal Q_\ell$, define
\begin{equation}\label{eq:in-memory-confidence-update}
\lambda_\ell
:=
\max_{P\in\mathcal Q_\ell}
(\widehat\mu_{P,1}-a_\ell),
\qquad
\mathcal C_\ell
:=
\{P\in\mathcal Q_\ell:
\widehat\mu_{P,2}+a_\ell\ge\lambda_\ell-u_\ell\}.
\end{equation}
If $\mathcal Q_\ell$ is empty, set $\mathcal C_\ell=\varnothing$
and use a fixed fallback action.
Of the first-sweep estimates, retain only their largest
lower confidence bound $\lambda_\ell$.
The second sweep computes the new indicator vector
one entry at a time.
Both sampling sequences depend only on $\mathcal C_{\ell-1}$,
so the two sweeps form one batch.
On the common confidence event, $\mathcal C_\ell$ is safe
at scale $u_\ell$.

At level $\ell$, there are $O_d(u_\ell^{-d})$ blocks, each with
$O_d(b_{B,r}u_\ell^{-2})$ pulls of gap $O_d(u_{\ell-1})$.
Choose the intermediate scales so that
$u_{\ell-1}u_\ell^{-d-2}\asymp_d s^{-\gamma_d(L)}$,
where $\gamma_d(L):=(d+1)/[1-(d+2)^{-L}]$.
\begin{lemma}[Hierarchical upper bound]
\label{lem:in-memory-branch}
For some $C_d>0$ and all sufficiently large $T$, if
\begin{equation}\label{eq:in-memory-fixed-resources}
W\ge C_d(s^{-d}+b_{B,r}),\qquad
C_dLb_{B,r}r^{-d-2}\le T/2,
\end{equation}
then a policy $\mathcal A\in\mathfrak A_{B,W}$ with fixed
batch boundaries and at most $L+2\le B$ batches satisfies
\begin{equation}\label{eq:in-memory-fixed-scale-regret}
\sup_{f\in\Lip_1(\X),\,\nu\in\Dclass(f)}
\E_{\nu,\mathcal A}\bigl[\Reg_T(f)\bigr]
\le C_d\left[
Lb_{B,r}s^{-\gamma_d(L)}+b_{B,r}sr^{-d-2}+Tr
\right].
\end{equation}
\end{lemma}
When $L=B-2$, balancing these costs gives exponent
$\beta_{d,B}$ if the indicator vector fits in memory;
otherwise, we increase $s$.

\paragraph{Completing the upper bound.}
Use a fixed action for $B=1$ and refinement from the full domain
for $B=2$.
For $B\ge3$, Appendix~\ref{app:upper-assembly} selects a
construction according to its batch requirement and chooses
feasible spatial scales to prove \eqref{eq:joint-frontier-upper}.
The construction need not change when the dominant regret
term changes.
For example, at $B=3$, the hierarchy attains the joint term
with logarithmic memory and the batch-depth term with
sufficiently large memory, up to logarithmic factors.

\subsection{Comparison with the finite-arm guarantee}
\label{app:finite-arm-comparison}

We compare our upper bound with the guarantee obtained by
applying \citet[Theorem~2]{huang2026few} to a regular grid
and optimizing its size under the same budgets.
Under the conditions in Appendix~\ref{app:finite-arm-calculation},
when the joint term dominates, the two bounds have polynomial orders
\[
\begin{aligned}
\text{Direct grid:}\quad&
T^{\alpha_d}
\left(\frac{T^{d/(d+2)}}{BW}\right)^{1/d},\qquad
\text{Our bound:}\quad&
T^{\alpha_d}
\left(\frac{T^{d/(d+2)}}{BW}\right)^{1/[d(d+3)]}.
\end{aligned}
\]
Thus our bound is polynomially smaller when
$T^{d/(d+2)}/(BW)$ grows polynomially with $T$.

\section{Minimax regret with one bit of memory}
\label{sec:timing-sensitivity}

For $W\gtrsim_d\log(eT)$, fixed batch boundaries attain
minimax regret up to logarithmic factors.
We now consider $W=1$, the smallest nonzero memory budget,
to determine the attainable regret and whether fixed
boundaries remain sufficient.
Let $\mathfrak R_T^{\mathrm{fix}}(B,W)$ denote minimax regret
over policies with batch boundaries fixed in advance.

\begin{theorem}[One-bit minimax regret]
\label{thm:one-bit-timing}
There exist constants $c_d,C_d>0$, depending only on $d$,
such that, for every $T\ge2$ and all $B$,
\begin{equation}\label{eq:one-bit-frontier}
 c_dTB^{-1/(d+2)}
 \le \mathfrak R_T(B,1)
 \le C_dT\left(\frac{\log(eB)}{B}\right)^{1/(d+2)}.
\end{equation}
In contrast, $\mathfrak R_T^{\mathrm{fix}}(B,1)=\Theta_d(T)$ for every $B$.
\end{theorem}

For fixed boundaries, conditioning on the public seed leaves
at most two possible actions at each round when $W=1$;
Lemma~\ref{lem:finite-action-codebook} therefore gives
linear regret.

At $B=T$, adaptive boundaries attain near-optimal regret
$\widetilde\Theta_d(T^{\alpha_d})$, whereas fixed boundaries
still incur linear regret.
Even with adaptive boundaries, the three-term lower bound
need not be tight below logarithmic memory.
For $d=1$ and $B\asymp T^{1/2}$, its value at $W=1$ is
of order $T^{2/3}$, whereas the one-bit minimax regret is
$\widetilde\Theta(T^{5/6})$.

\paragraph{Adaptive boundaries.}
The policy chooses batch lengths so that the comparison
statistics and the incumbent's index can be computed from
the current round and batch indices.
A one-round query starting at time $t$ stores a bit $M_{t+1}$.
The next batch pulls the incumbent for $1+M_{t+1}$ rounds
and resets the state, so the following batch starts at
$t+2+M_{t+1}$.
All these rounds count toward the horizon and regret;
the starting times remain determined by the seed and
boundary states.
This is a mutable-state guarantee, not a bound on total physical
storage: the freely available timing indices can retain
reward-dependent information after the bit is reset.
For dyadic $0<\rho\le1/2$, Lemma~\ref{lem:one-bit-fixed-accuracy}
gives a policy with $O_d(\rho^{-d-2}\log(e/\rho))$ exploration
rounds and batches and expected regret
$O_d(\rho^{-d-1}\log(e/\rho)+T\rho)$,
provided exploration finishes by $T/2$.
For sufficiently large $B$, choose
$\rho\asymp_d(\log(eB)/B)^{1/(d+2)}$ with a large enough
constant to use at most $B/2\le T/2$ exploration rounds and batches.
A final exploitation batch gives $O_d(T\rho)$ regret;
for bounded $B$, use a fixed action.

\begin{proof}[Proof sketch of the lower bound]
Fix a one-bit policy.
Consider a constant reference mean and
$K\asymp_d\delta^{-d}$ disjoint radius-$\delta$ bumps
of height $\delta$.
Let $P_0,P_1,\ldots,P_K$ denote the corresponding laws
of the public seed and nonterminal boundary states.
Conditional on a common boundary history and seed,
the squared Hellinger distances between the reference
and alternative laws of the next boundary state sum
to $O(\delta^2)$ over the alternatives, uniformly in
the batch length.
Composing over boundaries gives
\(
\sum_{v=1}^K H^2(P_0,P_v)\lesssim(B-1)\delta^2.
\)
Hence
$K^{-1}\sum_{v=1}^K\operatorname{TV}(P_0,P_v)
\lesssim_d\sqrt{B\delta^{d+2}}$.
Taking $\delta=c_dB^{-1/(d+2)}$ with sufficiently small
$c_d$ makes this average distance at most $1/4$.
Since the action sequence is determined by the seed and
boundary states, disjointness implies that some alternative
receives $\Omega(T)$ expected pulls outside its bump,
giving $\Omega_d(T\delta)$ regret.
See Appendix~\ref{app:one-bit-timing}.
\end{proof}

\section{Conclusion}\label{sec:conclusion}

Lipschitz geometry allows coarse regional information to guide
fine action comparisons at low regret.
The one-bit separation shows that batch timing is part of
the memory model: adaptive boundaries can preserve information
through the freely available round index.

Two questions remain open: the sharp regret dependence for
$2\le W=o(\log T)$ under fixed and adaptive boundaries,
and whether restricted Lipschitz classes with fewer near-optimal
regions admit sharper tradeoffs.

{\small
\setlength{\bibsep}{2pt plus 1pt minus 1pt}
\bibliographystyle{plainnat}
\bibliography{references}
}

\clearpage
\appendix
\setlength{\abovedisplayskip}{6pt plus 2pt minus 2pt}
\setlength{\belowdisplayskip}{6pt plus 2pt minus 2pt}
\setlength{\abovedisplayshortskip}{3pt plus 2pt minus 1pt}
\setlength{\belowdisplayshortskip}{4pt plus 2pt minus 1pt}
\setlength{\textfloatsep}{8pt plus 2pt minus 2pt}
\setlength{\floatsep}{7pt plus 2pt minus 2pt}
\numberwithin{figure}{section}
\renewcommand{\thefigure}{\Alph{section}.\arabic{figure}}
\renewcommand{\theHfigure}{\Alph{section}.\arabic{figure}}
\setcounter{figure}{0}
\numberwithin{algorithm}{section}
\renewcommand{\theHalgorithm}{\Alph{section}.\arabic{algorithm}}
\setcounter{algorithm}{0}

\section{Policy maps and boundary-state reconstruction}
\label{app:boundary-memory}

\paragraph{Policy maps.}
We complete the policy specification in
Section~\ref{subsec:policy-model}.
Let $(\Omega_0,\mathcal G_0)$ be the standard Borel space of the public seed $\omega$.
For $b\in\{0,\ldots,B-1\}$ and $u\in[T]$, the policy maps
are Borel on the following full domains:
\[
\begin{aligned}
\iota&:\Omega_0\longrightarrow\{0,1\}^W,\\
\psi_b&:
\Omega_0\times\{0,\ldots,T\}\times\{0,1\}^W
\longrightarrow\{0,\ldots,T\},\\
\phi_{b,u}&:
\Omega_0\times\{0,\ldots,T\}\times\{0,1\}^W
\longrightarrow\X,\\
\mathcal U_{b,u}&:
\Omega_0\times\{0,1\}^W\times\X\times[0,1]
\longrightarrow\{0,1\}^W.
\end{aligned}
\]
For every $(w,t,m)$, require
\[
\psi_b(w,t,m)\in
\begin{cases}
\{t+1,\ldots,T\},&t<T,\\
\{T\},&t=T,
\end{cases}
\qquad
\psi_{B-1}(w,t,m)=T.
\]
Thus batches are nonempty until the horizon is reached,
and the last permitted batch ends at $T$.
\label{app:static-classes}
The sequential finite-state subclass has $B=T$ and $\tau_b=b$.

\paragraph{Boundary states.}
Let $\widehat B:=\min\{b\in[B]:\tau_b=T\}$ be the number
of nonempty batches.
For $0\le b\le B$, set
$\bar\tau_b:=\tau_{\min\{b,\widehat B\}}$, and define
\[
\bar M_b:=
\begin{cases}
M_{\tau_b},&b<\widehat B,\\
0^W,&b\ge\widehat B,
\end{cases}
\qquad b=1,\ldots,B-1.
\]
Write
\[
\mathbf M:=(\bar M_1,\ldots,\bar M_{B-1}),
\qquad
\Tr:=(\bar\tau_{1:B-1},A_{1:T}).
\]
The tuple $\mathbf M$ contains the states at nonterminal
batch boundaries, followed by fixed padding.
It is used only in the analysis and is not additional storage
available to the learner.
The transcript $\Tr$ records the batch boundaries and actions.

\begin{lemma}[Boundary reconstruction]
\label{lem:boundary-factorization}
For every policy, there is a policy-dependent Borel map such that
\[
\Tr=\RecMap_{\mathcal A}(\omega,\mathbf M)
\quad\text{a.s.}
\]
Conditional on the public seed, the transcript takes at most
$2^{(B-1)W}$ values.
More generally, the boundaries and actions through batch $j$
are determined by $\omega,\bar M_1,\ldots,\bar M_{j-1}$
and take at most $2^{(j-1)W}$ values.
\end{lemma}

\begin{proof}
For an input $(w,m_1,\ldots,m_{B-1})$, initialize
$t_0=0$ and $m_0=\iota(w)$, and recursively set
\[
t_{b+1}=\psi_b(w,t_b,m_b),\qquad
a_u=\phi_{b,u}(w,t_b,m_b),
\quad t_b<u\le t_{b+1},
\quad b=0,\ldots,B-1.
\]
The boundary conventions ensure $t_B=T$.
Finite composition of the policy maps defines the Borel map
$\RecMap_{\mathcal A}(w,m_{1:B-1})
=(t_{1:B-1},a_{1:T})$.

For the policy's padded boundary states, induction using
\eqref{eq:committed-endpoint-map}--\eqref{eq:committed-action-map}
gives $t_b=\bar\tau_b$ and $a_u=A_u$.
The recursion through batch $j$ uses only the first $j-1$
boundary states, which have at most $2^{(j-1)W}$ possible values.
\end{proof}

\paragraph{Information bound.}
Let $V$ be independent of $\omega$, and let $U$ be generated
from $\Tr$ using fresh randomness independent of
$(V,\omega,\mathbf M,\Tr)$.
Conditional data processing gives
\begin{equation}\label{eq:boundary-information}
I(V;U\mid\omega)
\le I(V;\Tr\mid\omega)
\le H(\mathbf M\mid\omega)
\le(B-1)W.
\end{equation}
Entropy and mutual information are measured in bits
\citep[Chapters~2--3]{cover2006elements}.
For $B=1$, the tuple $\mathbf M$ is empty and the bound is zero.

\paragraph{Finite action ranges.}
\label{app:boundary-codebook}

\begin{lemma}[Lower bound from finite action ranges]
\label{lem:finite-action-codebook}
Let $S\ge1$ be an integer.
Suppose a policy $\mathcal A$ admits measurable sets
$\mathcal C_t(\omega)\subseteq\X$, determined by its public seed,
such that $|\mathcal C_t(\omega)|\le S$ and
$A_t\in\mathcal C_t(\omega)$ almost surely for every $t\in[T]$
and every admissible instance.
Then
\[
\sup_{f\in\Lip_1(\X),\,\nu\in\Dclass(f)}
\E_{\nu,\mathcal A}\bigl[\Reg_T(f)\bigr]\ge c_dTS^{-1/d}.
\]
\end{lemma}

\begin{proof}
Choose $N=2S$ centers $z_1,\ldots,z_N\in[1/4,3/4]^d$
separated by $4r_0$, where $r_0=c_dS^{-1/d}$ for a sufficiently
small constant $c_d>0$.
Consider the Bernoulli instances with means
$g_v(x)=1/4+(r_0-\|x-z_v\|_\infty)_+$.
These means belong to $\Lip_1(\X)$, and their disjoint regions
$G_v=\{x\in\X:\|x-z_v\|_\infty\le r_0/2\}$ satisfy
$\Delta_{g_v}(x)\ge(r_0/2)\mathbf1\{x\notin G_v\}$.

Write $\PP_v^\circ,\E_v^\circ$ for the corresponding interaction
laws and expectations.
For each $t$ and almost every seed $w$, at most $S$ regions
intersect $\mathcal C_t(w)$, so
\[
\frac1N\sum_{v=1}^N
\PP_v^\circ(A_t\in G_v\mid\omega=w)
\le
\frac1N\sum_{v=1}^N
\mathbf1\{\mathcal C_t(w)\cap G_v\ne\varnothing\}
\le\frac12.
\]
Integrating over the common seed law and summing over $t$ gives
\[
\frac1N\sum_{v=1}^N\E_v^\circ\Reg_T(g_v)
\ge\frac{r_0T}{4}
\ge c_dTS^{-1/d},
\]
which lower-bounds the supremum.
\end{proof}

By Lemma~\ref{lem:boundary-factorization}, for each seed the
reconstruction map provides an instance-independent set of at most
$2^{(B-1)W}$ action transcripts. Projecting onto the time-$t$ action and applying
Lemma~\ref{lem:finite-action-codebook} gives
\begin{equation}\label{eq:boundary-codebook-lower}
\mathfrak R_T(B,W)\ge c_dT2^{-(B-1)W/d}.
\end{equation}
\section{Lower-bound proofs}\label{app:lower}

We use the model and notation of Section~\ref{sec:prelim}, together
with the entropy conventions of
Appendix~\ref{app:boundary-memory}. Relative entropy is measured in nats.
Write $B_\infty(x,a):=\{y:\|x-y\|_\infty\le a\}$.
For measurable $G\subseteq\X$, write
$N_G(t):=\sum_{u=1}^t\mathbf1\{A_u\in G\}$.

\subsection{Hard instances and local testing}
\label{app:hard-family-geometry}

\paragraph{Regional packing.}
We realize the packing used in
Definition~\ref{def:routing-hard-family} as follows.
Choose an even number $2m\asymp_d s^{-d}$ of centers from the first
of the following grids, pair them as $(x_{j,0},x_{j,1})_{j\in[m]}$,
and use the second grid for the neighborhood centers in each cell:
\[
\begin{aligned}
\mathcal U_s
&=\tfrac14\boldsymbol 1_d
  +8s\{0,\ldots,\lfloor(16s)^{-1}\rfloor\}^d,\\
\mathcal Z(x)
&=x-\tfrac s4\boldsymbol 1_d
  +4r\{0,\ldots,\lfloor s/(8r)\rfloor\}^d.
\end{aligned}
\]
Here $\boldsymbol 1_d$ is the all-ones vector. The coarse centers
lie in $[1/4,3/4]^d$ and are separated by at least $8s$.
Each $\mathcal Z(x)$ contains the same number
$q\asymp_d(s/r)^d$ of points, lies in $B_\infty(x,s/4)$, and
has separation at least $4r$.
Enumerate these points as $z_{j,a,k}$.
The cells lie in $\X$, and all neighborhoods are pairwise disjoint and
contained in their parent cells.

\paragraph{Geometric properties.}
The base function and each added bump have Lipschitz constant at
most $1/2$, so all means belong to $\Lip_1(\X)$ and lie in
$[1/4,3/4]$ for small enough $s_d$.
Every selected cell is an optimal plateau, with
$f_v^\star=1/4+s/4$.
The cells have mutual distance at least $6s$, so every unselected
cell has mean $1/4$.
The bump has height $r/2$ and a unique peak in a selected cell;
thus $(f^+_{v,j,k})^\star=1/4+s/4+r/2$, attained uniquely at
$z_{j,v_j,k}$, and
\begin{align}
\Delta_{f_v}(x)&=s/4,
  &&x\in P_{j,1-v_j},\label{eq:hard-unselected-gap}\\
\Delta_{f^+_{v,j,k}}(x)&\ge r/2,
  &&x\notin G_{j,v_j,k}.\label{eq:hard-selected-gap}
\end{align}
Moreover, $0\le f^+_{v,j,k}-f_v\le r/2$, and
\begin{equation}\label{eq:hard-local-support}
f^+_{v,j,k}(x)=f_v(x)\qquad(x\notin G_{j,v_j,k}).
\end{equation}

For a fixed policy, write $\PP_v,\E_v$ and
$\PP^+_{v,j,k},\E^+_{v,j,k}$ for the interaction laws and
expectations under the base and local instances, and abbreviate
$N_{j,a,k}(t):=N_{G_{j,a,k}}(t)$.

\paragraph{Change of measure.}
\label{app:stopped-local-comparison}
We use a change-of-measure argument stopped after a prescribed
number of visits to the perturbed region
\citep{mannor2004sample,kaufmann2016complexity};
related formulations appear in
\citet{simchowitz2017simulator,huang2026few}.

\begin{lemma}[Localized change-of-measure bound]
\label{lem:stopped-local-comparison}
Fix a policy and two instances
$\nu_i\in\Dclass(f_i)$, $i=0,1$, with induced laws $\PP_i$
and expectations $\E_i$.
Let $G\subseteq\X$ be measurable and suppose, for some
$\Delta>0$ and $\kappa\ge0$, that
\[
\nu_{1,x}=\nu_{0,x}\quad(x\notin G),
\qquad
\sup_{x\in G}D(\nu_{1,x}\Vert\nu_{0,x})\le\kappa,
\qquad
\inf_{x\notin G}\Delta_{f_1}(x)\ge\Delta.
\]
For every integer $1\le n<T$,
\begin{equation}\label{eq:stopped-local-exploration}
\PP_0(N_G(T)\ge n)
\ge
1-\frac{\E_1\Reg_T(f_1)}{\Delta(T-n)}
-\sqrt{\frac{n\kappa}{2}}.
\end{equation}
\end{lemma}
\begin{proof}
For $t=0,\ldots,T$, let
$\mathcal F_t:=\sigma(\omega,(A_u,Y_u)_{1\le u\le t})$ and set
$\tau_n:=\inf\{t\in[T]:N_G(t)\ge n\}\wedge T$,
with $\inf\varnothing=\infty$.
The policy maps in Section~\ref{subsec:policy-model} make $A_t$
$\mathcal F_{t-1}$-measurable, so the stopped chain rule gives
\begin{equation}\label{eq:stopped-chain-rule-calculation}
D\!\left(
\left.\PP_1\right|_{\mathcal F_{\tau_n}}
\,\middle\Vert\,
\left.\PP_0\right|_{\mathcal F_{\tau_n}}
\right)
=
\E_1\sum_{t=1}^{\tau_n}
D(\nu_{1,A_t}\Vert\nu_{0,A_t})
\le \kappa\E_1N_G(\tau_n)
\le n\kappa.
\end{equation}
The event $E_n:=\{N_G(T)\ge n\}$ equals
$\{N_G(\tau_n)\ge n\}$ and therefore belongs to
$\mathcal F_{\tau_n}$.
Pinsker's inequality yields
$\PP_0(E_n)\ge\PP_1(E_n)-\sqrt{n\kappa/2}$.

On $E_n^\complement$,
$\Reg_T(f_1)\ge\Delta(T-N_G(T))\ge\Delta(T-n)$.
Hence
$\PP_1(E_n^\complement)
\le\E_1\Reg_T(f_1)/[\Delta(T-n)]$.
Combining these inequalities proves the claim.
\end{proof}

\begin{corollary}[Testing selected neighborhoods]
\label{lem:selected-probe-testing}
For sufficiently small $s_d$, there exist constants
$c_0,c_d,C_d>0$ such that, whenever
$0<s\le s_d$, $0<r\le s/16$, and $Tr^2\ge C_d$,
the threshold $n=\lfloor c_0r^{-2}\rfloor$ satisfies
$2\le n\le T/4$ and $n\asymp r^{-2}$.
For every fixed policy and all $v,j,k$,
if $\E^+_{v,j,k}\Reg_T(f^+_{v,j,k})\le c_dTr$, then
\[
\PP_v\bigl(N_{j,v_j,k}(T)\ge n\bigr)\ge\frac58.
\]
\end{corollary}

\begin{proof}
For $u,v\in[1/4,3/4]$,
$D(\operatorname{Ber}(u)\Vert\operatorname{Ber}(v))
\le16(u-v)^2/3$.
The local means differ by at most $r/2$ and agree outside the
selected neighborhood, whose complement has alternative gap
at least $r/2$.
Lemma~\ref{lem:stopped-local-comparison}, setting
$\kappa=4r^2/3$ and $\Delta=r/2$ gives:
\begin{equation}\label{eq:selected-probe-verification}
\PP_v(N_{j,v_j,k}(T)\ge n)
\ge
1-\frac{2\E^+_{v,j,k}\Reg_T(f^+_{v,j,k})}{r(T-n)}
-\sqrt{2nr^2/3}
\ge\frac58.
\end{equation}
Choose $c_0,c_d$ sufficiently small for the last inequality;
decreasing $s_d$ and increasing $C_d$ ensures the stated
bounds on $n$.
\end{proof}
\subsection{Sequential and joint lower bounds}
\label{app:information-proof}

\paragraph{Sequential minimax lower bound.}
Fix $s=s_d/2$ and $r=a_dT^{-1/(d+2)}$, where $a_d>0$
is a sufficiently small constant.
For sufficiently large $T$, these scales satisfy the assumptions
of Corollary~\ref{lem:selected-probe-testing}.
If a policy has expected regret below $c_dTr$ on every base
instance and local alternative, that corollary forces
$\Omega(r^{-2})$ expected visits to each selected neighborhood.
Since these neighborhoods are disjoint, under any base instance $f_v$,
\[
T
\ge \sum_{j=1}^m\sum_{k=1}^q
\E_v N_{j,v_j,k}(T)
\gtrsim_d mq r^{-2}
\asymp_d r^{-d-2}
= a_d^{-(d+2)}T.
\]
Choosing $a_d$ sufficiently small gives a contradiction.
Thus, for all sufficiently large $T$ and every admissible $B,W$,
\begin{equation}\label{eq:appendix-evidence-floor}
\mathfrak R_T(B,W)\ge c_dT^{\alpha_d}.
\end{equation}

\begin{proof}[Proof of Lemma~\ref{lem:routing-memory}]
Fix a policy and use the notation and estimator $\widehat V$
from the proof sketch in Section~\ref{subsec:regional-information},
with $n$ specified in Corollary~\ref{lem:selected-probe-testing}.
Suppose
\begin{equation}\label{eq:information-small-regret-assumption}
\mathcal R<c_d\min\{Tr,sr^{-d-2}\}
\end{equation}
for sufficiently small $c_d$.
The tie-breaking bits are independent of
$(V,\omega,A_{1:T})$.
Corollary~\ref{lem:selected-probe-testing} gives
$\E_v\rho_{j,v_j}\ge5/8$ for every $v,j$.
Since $0\le\rho_{j,v_j}\le1$, this implies
$\PP_v(\rho_{j,v_j}\ge1/4)\ge1/2$.
On the other hand, marking an unselected cell requires at least
$qn/4$ pulls of gap $s/4$, and these pulls are disjoint across cells.
Thus, pathwise,
\[
\Reg_T(f_v)\ge\frac{sqn}{16}
\sum_{j=1}^m\mathbf1\{\rho_{j,1-v_j}\ge1/4\}.
\]
Averaging and using $msqn\asymp_d sr^{-d-2}$ yields
\begin{equation}\label{eq:routing-decoder-asymmetry}
\frac1m\sum_j\Pbar(\rho_{j,V_j}\ge1/4)\ge\frac12,
\qquad
\Ebar\sum_j\mathbf1\{\rho_{j,1-V_j}\ge1/4\}
\le\frac{16\mathcal R}{sqn}\le\frac m8.
\end{equation}
The independent tie-breaking coins give
\[
\Pbar(\widehat V_j\ne V_j)
=\tfrac12\bigl[1-\Pbar(\rho_{j,V_j}\ge1/4)
                  +\Pbar(\rho_{j,1-V_j}\ge1/4)\bigr],
\]
so the average coordinate error is at most $5/16$.
Conditioning reduces entropy, and the binary error bound
and concavity of $h_2$ gives
\[
\begin{aligned}
H(V\mid\widehat V,\omega)
&\le\sum_j H(V_j\mid\widehat V_j)\\
&\le\sum_j h_2\!\left(\Pbar(\widehat V_j\ne V_j)\right)
\le m h_2(5/16).
\end{aligned}
\]
Since $V$ is uniform and independent of $\omega$,
$H(V\mid\omega)=m$.
The estimator $\widehat V$ depends only on the action sequence
and independent randomness, so conditional data processing gives
\begin{equation}\label{eq:routing-information-required}
[1-h_2(5/16)]m
\le I(V;\widehat V\mid\omega)
\le I(V;A_{1:T}\mid\omega).
\end{equation}
Since $m\asymp_d s^{-d}$, this proves
\eqref{eq:regret-information-bridge}.

Under \eqref{eq:routing-memory-condition},
\eqref{eq:boundary-information} gives
$I(V;A_{1:T}\mid\omega)\le(B-1)W\le c_ds^{-d}$,
contradicting \eqref{eq:regret-information-bridge}
for sufficiently small $c_d$.
Thus every policy has
$\mathcal R\ge c_d\min\{Tr,sr^{-d-2}\}$,
which proves \eqref{eq:two-scale-lower}.
\end{proof}

\begin{proof}[Completion of the sequential and joint lower bounds]
\phantomsection
\label{app:memory-optimization}
The sequential lower bound \eqref{eq:appendix-evidence-floor} holds for
all $B,W$. For bounded $(B-1)W$, \eqref{eq:boundary-codebook-lower}
gives linear regret and dominates the joint term.
In the intermediate range $K_d<(B-1)W<\eta_dT^{d/(d+2)}$, choose
$s\asymp_d(1+(B-1)W)^{-1/d}$ and
$r\asymp_d(s/T)^{1/(d+3)}$, allowing dyadic rounding.
A small coefficient in $s$ and large $K_d$ ensure $s\le s_d$ and
\eqref{eq:routing-memory-condition}.
Choosing $\eta_d$ small makes $Ts^{d+2}$ sufficiently large, giving
$r/s\asymp_d(Ts^{d+2})^{-1/(d+3)}\le1/16$.
Also $Tr^2\gtrsim_d T^{d/(d+2)}$, so the testing assumption holds
for sufficiently large $T$.
Applying \eqref{eq:two-scale-lower} gives
\[
\mathfrak R_T(B,W)
\ge c_d\min\{Tr,sr^{-d-2}\}
\asymp_d T^{\frac{d+2}{d+3}}
(1+(B-1)W)^{-\frac1{d(d+3)}}.
\]
For $(B-1)W\ge\eta_dT^{d/(d+2)}$, the last expression is
$O_d(T^{\alpha_d})$ and is covered by the sequential lower bound.
Thus, for sufficiently large $T$ and every admissible $B,W$,
\[
\mathfrak R_T(B,W)\ge c_d\left[
T^{\alpha_d}\vee
T^{\frac{d+2}{d+3}}
(1+(B-1)W)^{-\frac1{d(d+3)}}
\right].
\]
\end{proof}

\subsection{Lower bound from batching}
\label{app:batch-depth-compatibility}

\begin{lemma}[Bernoulli transfer of the adaptive-grid batch lower bound]
\label{lem:batch-depth-transfer}
Fix $d\ge1$. There are constants $c_d,T_d>0$ such that, for
$T\ge T_d$ and $2\le B\le T$,
\[
\mathfrak R_T^{\mathrm{bat}}(B)
\ge c_d\frac{T^{\beta_{d,B}}}{B^2}.
\]
The same bound holds for $\mathfrak R_T(B,W)$ for every integer $W\ge0$.
\end{lemma}
\begin{proof}
Fix $\mathcal A\in\mathfrak A_{B,\infty}$.
Feed the binary observations
$\widetilde Y_t=\mathbf1\{Y_t^{\mathrm G}\ge0\}$
from the Gaussian experiment to $\mathcal A$.
We obtain a policy $\mathcal A^\sharp$ with exactly $B$
nonempty batches, as required by
\citet[Section~1.2]{feng2022lipschitz}, by subdividing
the original batches without changing their actions.

If new batch $b$ starts at $t_{b-1}$ inside an original batch
ending at $\tau$, set
\[
t_b=\min\{\tau,T-(B-b)\}.
\]
At each $t_b$, process the newly revealed binary rewards
in order. Invoke the original decision rule if $t_b=\tau$;
otherwise retain the original boundary and remaining actions.
This preserves the simulated states at original boundaries
and hence the entire action sequence.

Each $t_b$ is chosen at $t_{b-1}$ and satisfies
$t_{b-1}<t_b\le T-(B-b)$.
A strict subdivision leaves exactly $B-b$ rounds, which
form singleton batches.
If no strict subdivision occurs, each new batch completes
an original batch.
Since the original policy uses at most $B$ batches,
the construction yields exactly $B$ nonempty batches.

Applying \citet[Theorem~8]{feng2022lipschitz} to
$\mathcal A^\sharp$ on the full-dimensional Gaussian family
$\mathcal F_{d,T,B}$ gives a mean $\mu$ such that
\begin{equation}\label{eq:feng-adaptive-grid-lower-bound}
\E_\mu^{\mathcal A^\sharp}\sum_{t=1}^T
(\mu^\star-\mu(A_t))
\ge \frac{T^{\beta_{d,B}}}{256B^2}.
\end{equation}

The source radii satisfy $0<r_j\le r_1=1/B\le1/2$.
Its base means are maxima of the floor $r_1/2$ and cones with
peak heights at most $r_1/2+r_j/16+r_B/16$
\citep[equations~(25)--(26)]{feng2022lipschitz}; the alternatives
in its Lemma~6 increase one peak by at most $3r_j/16$.
Consequently, uniformly over $T,B$ and all source means,
\[
0\le\mu(x)\le\frac{r_1}{2}
 +\max_j\left(\frac{r_j}{16}+\frac{r_B}{16}
                         +\frac{3r_j}{16}\right)
\le\frac{13r_1}{16}\le\frac{13}{32}.
\]
Let $\Phi,\phi$ be the standard normal distribution function and
density, and put $f_\mu:=\Phi\circ\mu$ and
$c_\Phi:=\min_{z\in[0,13/32]}\phi(z)>0$.
Since $\|\phi\|_\infty<1$, $f_\mu\in\Lip_1(\X)$ and
\begin{equation}\label{eq:gaussian-to-bernoulli-gap}
f_\mu^\star-f_\mu(x)\ge
c_\Phi(\mu^\star-\mu(x)).
\end{equation}
Under the source observation $Y_t^{\mathrm G}\sim\mathcal N(\mu(A_t),1)$,
the transformed reward has the Bernoulli kernel with mean
$f_\mu(A_t)$.
Thus the simulated original states, boundaries, and actions have
the same joint law as under $\mathcal A$ on $f_\mu$; only the
additional subdivision boundaries differ.
Combining \eqref{eq:feng-adaptive-grid-lower-bound} and
\eqref{eq:gaussian-to-bernoulli-gap} gives
\[
\E_{f_\mu}^{\mathcal A}\Reg_T(f_\mu)
\ge c_\Phi\E_\mu^{\mathcal A^\sharp}
       \sum_{t=1}^T(\mu^\star-\mu(A_t))
\ge\frac{c_\Phi}{256B^2}T^{\beta_{d,B}}.
\]
Take the supremum over admissible instances and the infimum over
$\mathcal A$. The finite-memory claim follows from
$\mathfrak A_{B,W}\subseteq\mathfrak A_{B,\infty}$.
\end{proof}

\begin{proof}[Proof of the lower bound in
Theorem~\ref{thm:joint-frontier}]
For $2\le B\le T$, the sequential and joint bounds established
in Appendix~\ref{app:memory-optimization}, together with
Lemma~\ref{lem:batch-depth-transfer}, give the three terms
in \eqref{eq:joint-frontier-lower}.
Each bound holds for every $W\ge0$ in the same policy class
$\mathfrak A_{B,W}$.
Taking the minimum of their constants therefore yields
\eqref{eq:joint-frontier-lower}.

For $B=1$, \eqref{eq:boundary-codebook-lower} gives
$\mathfrak R_T(1,W)\ge c_dT$ for every $W\ge0$.
Since $\beta_{d,1}=1$, the displayed maximum becomes
\[
T^{\alpha_d}\vee T\vee T^{\frac{d+2}{d+3}}=T,
\]
so the same conclusion holds.

Finally, choose $T_d$ as the maximum of the horizon thresholds
required by the preceding bounds. Both $c_d$ and $T_d$
depend only on $d$, completing the proof.
\end{proof}
\section{Upper-bound constructions}\label{app:upper-constructions}

We specify and verify the constructions of Section~\ref{sec:upper-main},
using its notation for safe active sets and refinement.
Appendix~\ref{app:completion} selects their scales and completes
the upper-bound proof.

\subsection{Refinement within active cells}
\label{app:refinement}
\label{subsec:fixed-geometry}
\label{subsec:active-set-stream}

Throughout this appendix, use the fallback action $x_\circ$ from
Section~\ref{subsec:formal-frontier}; ties in estimated means or confidence
bounds retain the first action in the prescribed order.
Write $\operatorname{bl}(z):=2+\lceil\log_2(1\vee z)\rceil$ for
$z>0$, and let $\lceil z\rceil_2$ be the smallest dyadic number at
least $z\in(0,1]$.  Thus $\operatorname{bl}(z)\asymp\log(ez)$ for
$z\ge1$ and $z\le\lceil z\rceil_2<2z$.
Geometry, traversal orders, and numerical schedules are fixed from the
problem parameters and public seed, and regenerated without stored tables.
Retained addresses and statistics are charged to the state budget.

\paragraph{Dyadic grids.}
Partition each coordinate into half-open intervals of length $2^{-j}$,
including the right endpoint in the last interval, and use midpoint
representatives.  The level-$j'$ ancestor of a level-$j$ cell,
$j'\le j$, has address
\begin{equation}\label{eq:dyadic-general-ancestor}
\left(j',\left\lfloor k_1/2^{j-j'}\right\rfloor,\ldots,
             \left\lfloor k_d/2^{j-j'}\right\rfloor\right).
\end{equation}
To realize the partition and nets of Section~\ref{subsec:active-set-interface},
take $j_s=\lceil\log_2(1/s)\rceil$ and
$j_r=\lceil\log_2(1/r)\rceil$.
The level-$j_s$ cells form $\mathcal P_s$, with diameter at most $s$
and $|\mathcal P_s|\le(2/s)^d$.
Their level-$j_r$ descendant midpoints form $r$-nets
$\mathcal N_r(P)$ of size at most $(2s/r)^d$ per parent.
A level and $d$ coordinate indices give each cell an
$O_d(\log(e/r))$-bit address and allow lexicographic traversal
without a stored geometric table.

\paragraph{Rounded means.}
For a prescribed block length $n\ge1$ and mesh $0<\varepsilon\le1$,
store $S_t=\sum_{i\le t}\lfloor Y_i/\varepsilon\rfloor$ and return
$\widehat\mu=\varepsilon S_n/n$.
Since $0\le S_t\le n\lfloor1/\varepsilon\rfloor$, the accumulator uses
$\lceil\log_2(n\lfloor1/\varepsilon\rfloor+1)\rceil$ bits, and
averaging the rounding errors gives
\begin{equation}\label{eq:finite-precision-register}
\left|\widehat\mu-\frac1n\sum_{i=1}^nY_i\right|<\varepsilon.
\end{equation}
If the block samples an action $x$ measurable at its starting history
$\mathcal H$, the fresh bounded rewards have conditional mean $f(x)$.
Thus Hoeffding's inequality gives, for $a>\varepsilon$, almost surely,
\begin{equation}\label{eq:streaming-mean-conditional-tail}
\PP\bigl(|\widehat\mu-f(x)|>a\mid\mathcal H\bigr)
\le2\exp\{-2n(a-\varepsilon)^2\}.
\end{equation}

\paragraph{Confidence screening.}
Both constructions use the following two inequalities.
For $\eta,a,\zeta\ge0$, suppose $\mu_\star$ bounds the candidate means,
$\mu_\star-\eta\le\lambda\le\mu_\star$, and
$f(x)\le U(x)\le f(x)+2a$.
The rule $U(x)\ge\lambda-\zeta$ retains every candidate with
$\mu_\star-f(x)\le\zeta$, since
$U(x)\ge f(x)\ge\mu_\star-\zeta\ge\lambda-\zeta$.
Conversely, every retained candidate satisfies
\[
\mu_\star-f(x)\le\mu_\star-U(x)+2a\le\eta+\zeta+2a.
\]

\paragraph{Refinement across groups.}
For $\mathcal P_s$, fix pairwise disjoint, nonempty groups
$\mathcal I_1,\ldots,\mathcal I_{J_0}\subseteq[K]$.
Before refining group $j$, let $t_j$ be the current boundary and
$\mathcal H_j=\sigma(\omega,(A_t,Y_t)_{t\le t_j})$.
At $t_j$, the indicator vector $Z^{(j)}$ is stored in the
mutable state, and the filler
$\bar x_j\in\{x_\circ,x_1,\ldots,x_K\}$ is determined by
that state, the public seed, and the round and batch indices.
These requirements hold on every sample path.
Let $\mathcal G_j$ be an $\mathcal H_j$-measurable event.
The selected cells are
$\mathcal C_j=\{P_k:k\in\mathcal I_j,\ Z_k^{(j)}=1\}$.
On $\mathcal G_j$, assume
\[
\Delta_f(x)\le\kappa_ds
\qquad\text{for every }
x\in\{\bar x_j\}\cup\bigcup_{P\in\mathcal C_j}P.
\]
On $\mathcal G:=\bigcap_j\mathcal G_j$, the collection
$\bigcup_j\mathcal C_j$ contains a maximizer's cell.

Construction phases may occur between refinement groups,
provided they preserve the fine-action record.
The auxiliary memory bound below includes the filler index
and refinement workspace; any other state retained for later
construction is counted separately.

\begin{lemma}[Finite-memory refinement]
\label{lem:active-set-refinement-conditional}
Suppose $C_dr^{-d-2}\ell_r\le T/2$ and all construction and
refinement batches finish by round $T$ on every sample path.
Under the assumptions above, refinement uses one batch per
group, at most $C_dr^{-d-2}\ell_r$ pulls, and at most
$C_d\ell_r$ auxiliary bits beyond the stored indicator vector
after every observation.
The resulting policy $\mathcal A$ satisfies
\begin{equation}\label{eq:refinement-good-event-contract}
\E_{\nu,\mathcal A}
\bigl[\Reg_{\mathrm{ref+exp}}(f)\mathbf1_{\mathcal G}\bigr]
\le C_d\bigl[sr^{-d-2}\ell_r+Tr\bigr].
\end{equation}
\end{lemma}

\begin{proof}
\emph{Sampling and memory.}
Set $n_r=\lceil A_d^{\mathrm{ref}}r^{-2}\ell_r\rceil$
and $\varepsilon=r/512$.
Reserve one $n_r$-pull block for each
$x\in\mathcal N_r(P_k)$, $k\in\mathcal I_j$.
Sample $x$ if $Z_k^{(j)}=1$ and $\bar x_j$ otherwise.
The indicator vector and filler index determine all actions
at the batch start.

Process the active-block rounded means one at a time,
retaining only the largest mean and its action.
Discard each accumulator after comparison and preserve this
record across groups and intervening construction phases.
Finally, exploit the recorded action $\widehat x$, or
$x_\circ$ if the record is empty.
The number of blocks is
\[
N_{\mathrm{block}}
:=\sum_j\sum_{k\in\mathcal I_j}|\mathcal N_r(P_k)|
\le C_dr^{-d},
\qquad
N_{\mathrm{block}}n_r\le C_dr^{-d-2}\ell_r.
\]
The filler index uses at most $\lceil\log_2(K+1)\rceil
=O_d(\ell_r)$ bits.
Together with the grid addresses, control indices, one
accumulator, and fine-action record, the refinement workspace
uses $O_d(\ell_r)$ auxiliary bits.

\emph{Accuracy and regret.}
Let $\mathcal E$ be the event that every active-block estimate
has error at most $r/32$.
Conditional on its starting history, each active block
satisfies \eqref{eq:streaming-mean-conditional-tail}.
Taking expectations and a union bound gives
\[
\PP(\mathcal E^\complement)
\le 2N_{\mathrm{block}}
\exp\!\left\{-2n_r(r/32-\varepsilon)^2\right\}
\le r
\]
for sufficiently large $A_d^{\mathrm{ref}}$.

On $\mathcal G$, every refinement pull has gap at most
$\kappa_ds$, and a maximizer's cell supplies a sampled fine
action with gap at most $r$.
On $\mathcal G\cap\mathcal E$, the recorded action therefore
satisfies $\Delta_f(\widehat x)\le r+2(r/32)\le2r$.
Hence
\[
\begin{aligned}
\E\bigl[\Reg_{\mathrm{ref+exp}}(f)\mathbf1_{\mathcal G}\bigr]
&\le \kappa_dsN_{\mathrm{block}}n_r
     +2Tr+T\PP(\mathcal G\cap\mathcal E^\complement)\\
&\le C_d\bigl[sr^{-d-2}\ell_r+Tr\bigr].
\end{aligned}
\]
\end{proof}

Taking every $\mathcal G_j$ to be the whole sample space gives
the unconditional refinement guarantee in
Section~\ref{subsec:active-set-interface}.
Under the same hypotheses, using $\Reg_T(f)\le T$ on
$\mathcal G^\complement$ gives
\begin{equation}\label{eq:construction-refinement-contract}
\E\Reg_T(f)
\le\E[\Reg_{\mathrm{build}}(f)\mathbf1_{\mathcal G}]
  +C_d\bigl[sr^{-d-2}\ell_r+Tr\bigr]
  +T\PP(\mathcal G^\complement).
\end{equation}

\subsection{Active-set construction in groups}
\label{app:serialized}
\label{subsec:serialized-instantiation}

We prove Lemma~\ref{lem:two-scale-upper}
using the following sampling schedule.

\paragraph{Parameters and updates.}
Use $f_K^\star$ from Section~\ref{subsec:serialized-main}.
Prescribe
\begin{equation}\label{eq:serialized-parameters}
\begin{gathered}
n_s=\lceil A_d^{\mathrm{ser}}s^{-2}\ell_T\rceil,\qquad
L_{\mathrm{ser}}=\lceil\log_2\log_2(4n_s)\rceil\vee1,\\
w_{\mathrm{ctl}}=\lceil C_d^{\mathrm{ctl}}\ell_T\rceil.
\end{gathered}
\end{equation}
Let $w_s:=\min\{K,W-w_{\mathrm{ctl}}\}\ge1$,
$J_s:=\lceil K/w_s\rceil$, and partition $[K]$ into fixed groups
$\mathcal I_1,\ldots,\mathcal I_{J_s}$ of size at most $w_s$.
Each group uses $i$ batches in pass $i$ and one final refinement
batch. Set $D_{\mathrm{ser}}:=L_{\mathrm{ser}}(L_{\mathrm{ser}}+1)/2+1$;
including exploitation, the batch constraint is
\begin{equation}\label{eq:serialized-throughput-law}
J_sD_{\mathrm{ser}}+1\le B,
\qquad
K\le w_s\left\lfloor\frac{B-1}{D_{\mathrm{ser}}}\right\rfloor.
\end{equation}
For $c_\diamond\ge64$, use
\begin{align*}
n_0&=1,
&c_\ell&=2^{\ell-L_{\mathrm{ser}}+1},\\
n_\ell&=\left\lceil c_\ell\sqrt{n_sn_{\ell-1}}\right\rceil,
&&1\le\ell\le L_{\mathrm{ser}},\\
\varepsilon&=2^{-10-\lceil\frac12\log_2n_s\rceil},
&a_\ell&=\sqrt{2\log(c_\diamond KL_{\mathrm{ser}}^2/r)/n_\ell}
               +2\varepsilon.
\end{align*}

In pass $i$, the reference action is $x_{\bar k}$ with
$\bar k=k_{i-1}$.
Let $\bar\mu_i$ be the rounded mean of its $n_i$-pull prefix
in the first batch, and set $\lambda_i=\bar\mu_i-a_i$.
The reference action is fixed throughout the pass, and
$\lambda_i$ is fixed after this prefix.

At step $\ell$, the current indicator vector prescribes
$n_\ell$ pulls of $x_k$ for each active index $k$,
and of $x_{\bar k}$ otherwise.
For each active candidate, compute $\widehat\mu_{k,\ell}$
from its $n_\ell$ rewards using mesh $\varepsilon$, and set
$\LCB_{k,\ell}=\widehat\mu_{k,\ell}-a_\ell$ and
$\UCB_{k,\ell}=\widehat\mu_{k,\ell}+a_\ell$.
Update the indicator vector by
\eqref{eq:serialized-reference-update}.
Filler estimates are ignored; the first batch's candidate
actions are fixed before the reference prefix is observed.
In each nonfinal pass $i$, among candidates surviving step $i$,
store the largest lower confidence bound and its index in
$\mathrm{ch}_i=(k_i^{\mathrm{ch}},\lambda_i^{\mathrm{ch}})$.
In the final pass, $\mathrm{best}=(x^{\mathrm{best}},\widehat\mu^{\mathrm{best}})$
retains the fine-action record from Appendix~\ref{app:refinement}.
The sentinel $\lambda_i=\bot$ triggers the incumbent prefix once;
$\varnothing$ denotes an uninitialized action address.
Algorithm~\ref{alg:serialized-fixed-scales} specifies the traversal
and initialization.

\begin{algorithm}[htbp]
\caption{Policy processing groups at fixed scales}
\label{alg:serialized-fixed-scales}
\begin{algorithmic}[1]
\Require Scales $r\le s$, groups
$(\mathcal I_j)_{j\le J_s}$, and the prescribed sampling schedule.
\State Choose a public initial incumbent $k_0\in[K]$.
\For{$i=1,\ldots,L_{\mathrm{ser}}$}
    \State $(\bar k,\lambda_i)\gets(k_{i-1},\bot)$.
    \If{$i<L_{\mathrm{ser}}$}
        \State $\mathrm{ch}_i\gets(\varnothing,-\infty)$.
    \Else
        \State Replace the pass record by
        $\mathrm{best}\gets(\varnothing,-\infty)$.
    \EndIf
    \For{$j=1,\ldots,J_s$}
        \State Initialize the group indicator vector to all ones.
        \State Perform steps $1,\ldots,i$ using
        \eqref{eq:serialized-reference-update},
        keeping the reference action fixed and updating
the group indicators and record as above;
        denote the resulting indicator vector by $Z^{(j)}$.
        \If{$i=L_{\mathrm{ser}}$}
            \State Refine group $\mathcal I_j$ using $Z^{(j)}$
and filler $x_{\bar k}$, updating $\mathrm{best}$
as in Appendix~\ref{app:refinement}.
        \EndIf
        \State Erase $Z^{(j)}$.
    \EndFor
    \If{$i<L_{\mathrm{ser}}$}
        \State $k_i\gets k_i^{\mathrm{ch}}$ if initialized;
        otherwise $k_i\gets k_{i-1}$.
    \EndIf
\EndFor
\State Use $x^{\mathrm{best}}$ for all remaining pulls if
initialized, and $x_\circ$ otherwise.
\end{algorithmic}
\end{algorithm}

\begin{proof}[Proof of Lemma~\ref{lem:two-scale-upper}]
Fix the constants in the prescribed sampling schedule before choosing
$C_d$ in \eqref{eq:two-scale-upper-resources}.
Its memory condition then ensures $w_{\mathrm{ctl}}\le W/2$ and
$J_s\le C_d(1+s^{-d}/W)$.
The horizon condition and $r\le s$ imply
$\log n_s\le C_d\ell_T$, hence
$D_{\mathrm{ser}}\le C_d(1+\log\ell_T)^2$.
After increasing the lemma's constant, its batch condition gives
$J_sD_{\mathrm{ser}}+1\le B$.

\proofparagraph{Sampling and memory costs}
Choose $A_d^{\mathrm{ser}}$ so that $L_{\mathrm{ser}}\ge2$,
$\log_2n_s\ge2L_{\mathrm{ser}}$, and $L_{\mathrm{ser}}^2\le\sqrt{n_s}$.
Removing ceilings from the recurrence gives
$v_\ell=n_s^{1-2^{-\ell}}
2^{-2(L_{\mathrm{ser}}-\ell)+L_{\mathrm{ser}}2^{1-\ell}}$.
Its consecutive ratios are at least four, and
$v_\ell\le n_\ell\le2v_\ell$, yielding $n_\ell\ge2n_{\ell-1}$ and
\begin{equation}\label{eq:interpolation-terminal-bounds}
n_s/2\le n_{L_{\mathrm{ser}}}\le2n_s,\qquad
n_{L_{\mathrm{ser}}-1}\ge n_s/16,\qquad
n_\ell\le2n_s4^{-(L_{\mathrm{ser}}-\ell)}.
\end{equation}
Consequently,
\begin{align}
\sum_{i=1}^{L_{\mathrm{ser}}}\sum_{\ell\le i}n_\ell
&\le Cn_s,
\label{eq:interpolation-first-sum}\\
L_{\mathrm{ser}}n_1+
\sum_{i=1}^{L_{\mathrm{ser}}}\sum_{\ell=2}^{i}
\frac{n_\ell}{\sqrt{n_{\ell-1}}}
&\le C\sqrt{n_s}.
\label{eq:interpolation-level-one-sum}
\end{align}
For the second bound, use
$n_\ell/\sqrt{n_{\ell-1}}\le c_\ell\sqrt{n_s}+1$,
$\sum_{\ell=2}^{L_{\mathrm{ser}}}(L_{\mathrm{ser}}-\ell+1)c_\ell\le8$,
and $L_{\mathrm{ser}}^2\le\sqrt{n_s}$.
Candidate and filler pulls total $K\sum_i\sum_{\ell\le i}n_\ell$;
incumbent prefixes add $\sum_i n_i\le4n_s$.
Thus construction and refinement use
\begin{equation}\label{eq:serialized-tournament-pulls}
T_{\mathrm{build}}\le C_dKn_s\le C_ds^{-d-2}\ell_T,
\qquad T_{\mathrm{ref}}\le C_dr^{-d-2}\ell_r.
\end{equation}
The horizon hypothesis, with sufficiently large constant,
bounds the total number of construction and refinement pulls
by $T/2$.

The simultaneous state contains one group indicator vector, the fixed incumbent
and its LCB, one block accumulator with control indices, and one pass
record. In the final pass, the fine-action record replaces the pass
record and persists while later indicator vectors are recomputed.
Each estimate updates its record before the accumulator is erased.
The horizon condition gives $\log K+\log(1/r)\le C_d\ell_T$;
confidence offsets are prescribed, so the rounded-mean and address
bounds give at most
\begin{equation}\label{eq:serialized-tournament-space}
w_s+w_{\mathrm{ctl}}\le W
\end{equation}
bits after every observation, for sufficiently large $C_d^{\mathrm{ctl}}$.
Refinement rewards update only refinement state. They do not change
the reference action and its lower confidence bound or subsequent comparison updates.
Every later comparison block therefore retains its prescribed action
and sample size conditional on its starting history.

\proofparagraph{Confidence and safe active sets}
There are at most $2KL_{\mathrm{ser}}^2$ candidate and incumbent
comparison blocks. Apply \eqref{eq:streaming-mean-conditional-tail}
with mesh $\varepsilon$, each block's $n_\ell,a_\ell$ (or $n_i,a_i$
for a prefix), and failure budget $r/(2KL_{\mathrm{ser}}^2)$.
Their joint accuracy event $\mathcal G$ has probability at least
$1-r$ and satisfies
\begin{align}
f(x_{\bar k})-2a_i\le\lambda_i&\le f(x_{\bar k}),
\label{eq:serialized-reference-bracket}\\
f(x_k)-2a_\ell\le\LCB_{k,\ell}\le f(x_k)
&\le\UCB_{k,\ell}\le f(x_k)+2a_\ell.
\label{eq:serialized-confidence-brackets}
\end{align}
A representative attaining $f_K^\star$ survives every comparison,
and the pass record's LCB is at least its LCB. Hence
\begin{equation}\label{eq:serialized-incumbent-invariant}
f_K^\star-f(x_{k_i})\le2a_i,
\qquad 1\le i<L_{\mathrm{ser}}.
\end{equation}
The terminal sample bounds and horizon condition give
\begin{equation}\label{eq:serialized-final-radius}
a_{L_{\mathrm{ser}}-1}
\le C_d\sqrt{\log(c_\diamond KL_{\mathrm{ser}}^2/r)/n_s}
\le C_ds.
\end{equation}
For each final group $j$, let $\mathcal G_j$ be the accuracy
event for all comparison estimates through that group's
screening.
It is measurable before the group's refinement batch.
On $\mathcal G_j$, screening with
$\eta=2a_{L_{\mathrm{ser}}-1}+2a_{L_{\mathrm{ser}}}$
and slack $s$ gives, for $k\in\mathcal I_j$,
\begin{equation}\label{eq:serialized-final-mask-contract}
\mathbf1\{f_K^\star-f(x_k)\le s\}
\le Z_k^{(j)}
\le
\mathbf1\{f_K^\star-f(x_k)\le s+6a_{L_{\mathrm{ser}}-1}\},
\qquad k\in\mathcal I_j.
\end{equation}
The covering property gives $f^\star-f_K^\star\le s$,
so every point in a retained cell has gap $O_d(s)$.
The incumbent used for filler pulls satisfies the same bound.
A cell containing a maximizer has representative gap at most
$s$ relative to $f_K^\star$ and therefore survives its screening.
Thus $\mathcal G=\bigcap_j\mathcal G_j$ preserves a maximizer's
cell, while each $\mathcal G_j$ supplies the measurable safety
condition required by Appendix~\ref{app:refinement}.

\proofparagraph{Regret}
Let $\Reg_{\mathrm{build}}^K(f)$ denote construction regret relative to
$f_K^\star$. On $\mathcal G$, first-step pulls cost at most
$KL_{\mathrm{ser}}n_1$. For $i,\ell\ge2$, a sampled candidate
survived step $\ell-1$, so screening with
$\eta=2a_{i-1}+2a_i$ gives gap $C(a_{\ell-1}+\zeta_i)$.
Incumbent prefixes and fillers in pass $i\ge2$ have gap at most
$2a_{i-1}$; the first prefix is absorbed by the first-step cost.
Therefore \eqref{eq:interpolation-first-sum}--
\eqref{eq:interpolation-level-one-sum} yield
\begin{align*}
\Reg_{\mathrm{build}}^K(f)
&\le C_dK\left[L_{\mathrm{ser}}n_1+
\sum_{i=2}^{L_{\mathrm{ser}}}\sum_{\ell=2}^i n_\ell a_{\ell-1}
+s\sum_\ell n_\ell+\varepsilon n_s\right]\\
&\le C_dK\left[\sqrt{n_s\log(c_\diamond KL_{\mathrm{ser}}^2/r)}
+sn_s\right].
\end{align*}
Converting from $f_K^\star$ to $f^\star$ adds at most
$sT_{\mathrm{build}}$, so
\begin{equation}\label{eq:serialized-elimination-regret}
\Reg_{\mathrm{build}}(f)\le sT_{\mathrm{build}}+\Reg_{\mathrm{build}}^K(f)
\le C_ds^{-d-1}\ell_T.
\end{equation}
Substitute \eqref{eq:serialized-elimination-regret} into
\eqref{eq:construction-refinement-contract}.
Since $\PP(\mathcal G^\complement)\le r$,
$s^{-d-1}\le sr^{-d-2}$, and $\ell_r\le C_d\ell_T$,
we obtain \eqref{eq:serialized-fixed-scale}.
\end{proof}

\subsection{Hierarchical active-set construction}
\label{app:in-memory}

We use $\gamma_d(L)$ from Section~\ref{subsec:in-memory-main}.
For $3\le B\le T$ and dyadic $0<r\le s\le1/16$, prescribe
\[
L_B(s)=\min\left\{B-2,
\left\lfloor\frac{(d+1)\log(1/s)}{\log(d+2)}\right\rfloor\right\},
\qquad
b_{B,r}=\operatorname{bl}(A_d^{\mathrm{hier}}Br^{-d-1}),
\]
with $A_d^{\mathrm{hier}}$ sufficiently large.

\begin{proof}[Proof of Lemma~\ref{lem:in-memory-branch}]
\proofparagraph{Spatial levels}
Set $L=L_B(s)\ge1$ and $s=2^{-j_s}$.
Choose
$\theta_\ell=(1-(d+2)^{-\ell})/(1-(d+2)^{-L})$,
$j_\ell=\lceil\theta_\ell j_s\rceil$, and $u_\ell=2^{-j_\ell}$,
$0\le\ell\le L$.
Then $1=u_0\ge\cdots\ge u_L=s$.
The ideal radii $\widetilde u_\ell=s^{\theta_\ell}$ satisfy
$\widetilde u_{\ell-1}\widetilde u_\ell^{-d-2}
=s^{-\gamma_d(L)}$;
dyadic rounding changes each product by at most $2^{d+2}$.
Consequently,
\begin{equation}\label{eq:in-memory-equalized-sums}
\sum_{\ell=1}^L u_{\ell-1}u_\ell^{-d-2}
\le C_dLs^{-\gamma_d(L)},\qquad
\sum_{\ell=1}^L u_\ell^{-d-2}\le Ls^{-d-2}.
\end{equation}
Initialize the root cell as active.

\proofparagraph{One batch per spatial level}
Implement the two sweeps and update
\eqref{eq:in-memory-confidence-update} with
$n_\ell=\lceil A_d^{\mathrm{samp}}u_\ell^{-2}b_{B,r}\rceil$
reserved pulls per cell per sweep, rounding mesh $u_\ell/512$,
and confidence radius $a_\ell=u_\ell/32$.
Compute preceding-level ancestors by
\eqref{eq:dyadic-general-ancestor}.
Blocks outside $\mathcal Q_\ell$ use the first representative
of the preceding active set, or $x_\circ$ if it is empty.
If $\mathcal Q_\ell=\varnothing$, set $\lambda_\ell=-1$
and the new indicator vector to zero.
Both sweeps' actions are fixed by the old indicator vector;
first-sweep rewards affect only the second sweep's state updates.
Retain the old vector until the new vector is complete.

\proofparagraph{Validity and active-set construction cost}
There are at most $C_dLr^{-d}\le C_dBr^{-d}$ potential active-candidate
blocks.  Apply \eqref{eq:streaming-mean-conditional-tail} with the above
sample sizes and meshes, tolerance $u_\ell/32$, and conditional
failure budget $r/(2C_dBr^{-d})$ per block.
For sufficiently large $A_d^{\mathrm{samp}}$, their joint accuracy
event $\mathcal E_0$ has probability at least $1-r/2$ and is
measurable at the end of active-set construction.

On $\mathcal E_0$, induction preserves a maximizer's cell and
bounds every retained cell's gap at every point by $\kappa_du_\ell$.
Indeed, if $P_\ell^\star$ contains a maximizer and its ancestor
is active, then
$f^\star-u_\ell-2a_\ell\le\lambda_\ell\le f^\star$.
The screening rule, with
$\mu_\star=f^\star$, $\eta=u_\ell+2a_\ell$, and
$\zeta=u_\ell$, retains $P_\ell^\star$ and gives representative
gap at most $2u_\ell+4a_\ell$ for every survivor.
Adding the diameter gives gap at every point at most
$3u_\ell+4a_\ell\le\kappa_du_\ell$.
Thus every candidate and filler action at level $\ell$ lies in an
active preceding-level cell and has gap $O_d(u_{\ell-1})$.
Using \eqref{eq:in-memory-equalized-sums}, the construction regret
and pull count obey
\begin{equation}\label{eq:in-memory-total-narrowing-regret}
\Reg_{\mathrm{build}}(f)\mathbf1_{\mathcal E_0}
\le C_dLb_{B,r}s^{-\gamma_d(L)},\qquad
T_{\mathrm{build}}\le C_dLb_{B,r}s^{-d-2}.
\end{equation}

\proofparagraph{Refinement and resources}
Apply Lemma~\ref{lem:active-set-refinement-conditional}
to the terminal indicator vector as a single group $[K]$,
with validity event $\mathcal E_0$.
Choose the first active representative as filler, or $x_\circ$
if the active set is empty.
Its index is determined by the stored vector, and its safety
on $\mathcal E_0$ follows from the preceding argument.
Refinement and exploitation add two batches, so the count is
$L+2\le B$.
The simultaneously stored old and new indicator vectors require
$C_ds^{-d}$ bits; addresses, one accumulator, the reference,
and the fine-action record require $C_db_{B,r}$ bits.
The accumulator is erased after each estimate is processed.
Since $\ell_r\le C_db_{B,r}$, \eqref{eq:in-memory-fixed-resources}
bounds all construction and refinement pulls by $T/2$ and
ensures the state cap after every observation.
Finally, substitute
\eqref{eq:in-memory-total-narrowing-regret} and
$\PP(\mathcal E_0^\complement)\le r/2$
into \eqref{eq:construction-refinement-contract}
to obtain \eqref{eq:in-memory-fixed-scale-regret}.
\end{proof}

\section{Proof of the upper bound and consequences}
\label{app:completion}

We select feasible scales for the fixed-scale constructions of
Sections~\ref{app:serialized} and~\ref{app:in-memory} directly in the
proof of the joint upper bound, retaining its explicit $C_d\ell_T$ factor.

\subsection{Proof of the joint upper bound}
\label{app:upper-assembly}
\label{app:construction-scale-selection}

Define a horizon-dependent bound on the pass count and the
associated batch overhead:
\begin{equation}\label{eq:serialization-overhead}
 L_T^{\mathrm{ser}}=\lceil\log_2(8\ell_T)\rceil\vee1,
 \qquad
 \Gamma_T=L_T^{\mathrm{ser}}(L_T^{\mathrm{ser}}+1)/2+1
 \asymp(\log\log(eT))^2.
\end{equation}

\begin{proof}[Proof of the upper bound in Theorem~\ref{thm:joint-frontier}]
Throughout this proof, write $p:=d+2$.
Set $b_{B,T}:=\operatorname{bl}(8BT)$ and
$\upsilon_d:=1/[d(d+3)]$.
For $B=1$, a fixed action has regret at most $T=T^{\beta_{d,1}}$.
For the remaining cases, set
$B_0:=\lceil A_d^{\mathrm{switch}}\Gamma_T\rceil$.
We refine the full domain for $B=2$, use the hierarchy
for $3\le B<B_0$, and process cells in groups for $B\ge B_0$. The threshold $B_0$ accommodates the batch overhead of
processing cells in groups.
It need not coincide with the dominance boundaries in
Appendix~\ref{app:regime-comparisons}, which depend jointly
on $B$ and $W$.

Fix the constants in Lemmas~\ref{lem:two-scale-upper}
and~\ref{lem:in-memory-branch}, and let $H_d\ge1$ dominate
the resource coefficients in \eqref{eq:in-memory-fixed-resources}.
There is $J_d\ge1$, independent of the remaining choices, such that
for $B\le T$ and $r\ge T^{-1/p}$,
\[
b_{B,r}\le J_db_{B,T},\qquad b_{B,r}\le J_d\ell_T.
\]
These inequalities follow from the definition of $\operatorname{bl}$
and $r^{-d-1}\le T^{(d+1)/p}$.
Choose $A_d^{\mathrm{opt}}$ large enough for the resource bounds
of the construction in groups, then $A_d^{\mathrm{switch}}\ge1$ large enough
to absorb the additive batch overhead, and finally
\[
A_d^{\mathrm{scale}}
\ge4H_d\max\{1,J_d(A_d^{\mathrm{switch}}+1)\}.
\]
Choose the coefficient in $W\ge C_d\ell_T$ at least $4H_dJ_d$
and above the other memory thresholds, then increase $T_d$ as needed.
All constants depend only on $d$; subsequent unlabelled constants
may absorb these fixed choices.
The branch and its scales depend only on the known parameters
and public seed.

\proofparagraph{Two batches}
We first establish a bound valid for every integer $W\ge0$:
for sufficiently large $T$,
\begin{equation}\label{eq:two-batch-upper-curve}
\mathfrak R_T(2,W)
\le C_dT\min\left\{1,
\max\left\{
\left(\frac{\log(4T)}{T}\right)^{1/(d+3)},
2^{-c_dW}
\right\}\right\}.
\end{equation}
Use the safe scale-$1$ active set $\{\X\}$.
For a sufficiently large $A_d^{\mathrm{root}}$ and sufficiently
small $\xi_d>0$, let $r$ be the smallest dyadic radius at least
\[
\widetilde r:=
\max\left\{
\left(\frac{A_d^{\mathrm{root}}\log(4T)}{T}\right)^{1/(d+3)},
2^{-\xi_dW}
\right\}.
\]
If $r$ exceeds a fixed small cutoff, a fixed action proves
\eqref{eq:two-batch-upper-curve}.
Otherwise, $r<2\widetilde r$ and
$\ell_r\le C_d\log(4T)$.
Choosing $\xi_d$ and the cutoff sufficiently small, and
$A_d^{\mathrm{root}}$ sufficiently large, gives
\[
C_d\ell_r\le W,
\qquad
C_dr^{-d-2}\ell_r
\le \frac{C_d}{A_d^{\mathrm{root}}}Tr
\le T/2.
\]
Lemma~\ref{lem:active-set-refinement-conditional} applies
with one refinement batch followed by one exploitation batch,
yielding
\[
\mathfrak R_T(2,W)
\le C_d\bigl[r^{-d-2}\ell_r+Tr\bigr]
\le C_dTr.
\]
Substituting $r<2\widetilde r$ proves
\eqref{eq:two-batch-upper-curve}.

For $W\ge C_d\ell_T$, increasing the logarithmic-memory
constant ensures $T2^{-c_dW}\le T^{\alpha_d}$.
Since $\beta_{d,2}=(d+2)/(d+3)$,
\eqref{eq:two-batch-upper-curve} therefore gives
\[
\mathfrak R_T(2,W)
\le C_dT^{\beta_{d,2}}\ell_T^{1/(d+3)}
\le C_d\ell_T\frac{T^{\beta_{d,2}}}{2^2},
\]
which proves \eqref{eq:joint-frontier-upper} for $B=2$.

\proofparagraph{The case $3\le B<B_0$}

Suppose $3\le B<B_0$.  Define
\[
  \eta_T:=\frac{A_d^{\mathrm{scale}}\Gamma_Tb_{B,T}}{T},
  \qquad
  \gamma_\star:=\gamma_d(B-2),
  \qquad
  \vartheta:=\frac{p}{(p+1)\gamma_\star+1}.
\]
Then
\[
  \gamma_\star\vartheta=\beta_{d,B},
  \qquad
  1-\frac{1+\vartheta}{p+1}=\beta_{d,B},
  \qquad
  \vartheta\le\frac1p.
\]
Here $s_B$ is the radius dictated by active-set construction over $B-2$ spatial levels, $s_W$ is the
smallest radius whose complete indicator vector fits in memory, and $\widetilde r$
balances fine refinement with exploitation.  Set
\[
  s_B:=\eta_T^\vartheta,
  \qquad
  s_W:=\left(\frac{A_d^{\mathrm{scale}}}{W}\right)^{1/d},
  \qquad
  \widetilde s:=s_B\vee s_W,
  \qquad
  \widetilde r:=(\eta_T\widetilde s)^{1/(p+1)}.
\]
For $T$ large enough, $\eta_T\le1$.  Since
$p\vartheta\le1$ and $\widetilde s\ge s_B=\eta_T^\vartheta$,
\begin{equation}\label{eq:moderate-scale-feasibility}
  \eta_T\le \widetilde s^p,
  \qquad
  \widetilde r=(\eta_T\widetilde s)^{1/(p+1)}\le\widetilde s.
\end{equation}
After increasing the logarithmic-memory constant and $T_d$, uniformly over
$3\le B<B_0$ we also have $\widetilde s\le1/32$.  Take the smallest dyadic
$s\ge\widetilde s$ and $r\ge\widetilde r$.  Dyadic rounding is monotone, so
$0<r\le s\le1/16$.

By \eqref{eq:moderate-scale-feasibility},
$r\ge\widetilde r\ge\eta_T^{1/p}\ge T^{-1/p}$, so the fixed bounds with
coefficient $J_d$ apply.  Since $s\ge s_W$ and $W\ge C_d\ell_T$,
\[
 H_d(s^{-d}+b_{B,r})
 \le \frac{H_d}{A_d^{\mathrm{scale}}}W+H_dJ_d\ell_T
 \le \frac W4+\frac W4\le W.
\]
For the horizon condition, keep the switching coefficient explicit:
$L=L_B(s)\le B<B_0\le(A_d^{\mathrm{switch}}+1)\Gamma_T$ and
$r^{-p}\le\eta_T^{-1}$.  Consequently,
\begin{align*}
 H_dL b_{B,r}r^{-p}
 &\le H_dJ_d(A_d^{\mathrm{switch}}+1)
       \Gamma_Tb_{B,T}\eta_T^{-1}\\
 &=\frac{H_dJ_d(A_d^{\mathrm{switch}}+1)}
         {A_d^{\mathrm{scale}}}\,T
 \le T/4.
\end{align*}
Both resource inequalities in \eqref{eq:in-memory-fixed-resources} now hold
uniformly over $3\le B<B_0$, so Lemma~\ref{lem:in-memory-branch} applies.  Since
$r^{p+1}\asymp\eta_Ts$ and $b_{B,r}\le C_db_{B,T}$, its terminal refinement
term satisfies
\[
  b_{B,r}sr^{-p}
  \asymp \frac{b_{B,r}}{\eta_T}r
  \le C_dTr.
\]

If $L=B-2$, then $\gamma_d(L)=\gamma_\star$ and $s\ge s_B$, so
\[
  Lb_{B,r}s^{-\gamma_d(L)}
  \le C_d\Gamma_Tb_{B,T}\eta_T^{-\gamma_\star\vartheta}
  \le C_d T^{\beta_{d,B}}(\Gamma_Tb_{B,T})^{1-\beta_{d,B}}.
\]
If $L<B-2$, then
$L=\lfloor(p-1)\log(1/s)/\log p\rfloor$.
For $q=p^{-L}\le\min\{1/p,ps^{p-1}\}$,
\[
s^{-\gamma_d(L)}
=s^{-(p-1)}
\exp\!\left(\frac{(p-1)q\log(1/s)}{1-q}\right)
\le C_ds^{-(p-1)},
\]
since $s^{p-1}\log(1/s)$ is bounded.
Using $\eta_T\le s^p$ from \eqref{eq:moderate-scale-feasibility},
\[
\begin{aligned}
Lb_{B,r}s^{-\gamma_d(L)}
&\le C_d\Gamma_Tb_{B,T}s^{-(p-1)}
\le C_dT\eta_Ts^{-(p-1)}\\
&\le C_dT(\eta_Ts)^{1/(p+1)}
\le C_dTr.
\end{aligned}
\]
Finally, $s=s_B\vee s_W$ up to a factor two and
$r\asymp(\eta_Ts)^{1/(p+1)}$ imply
\[
  Tr\le C_d \max\left\{
    T^{\beta_{d,B}}(\Gamma_Tb_{B,T})^{1-\beta_{d,B}},
    T^{\frac{d+2}{d+3}}(\Gamma_Tb_{B,T})^{1/(p+1)}
      (1\vee W)^{-\upsilon_d}
  \right\}.
\]
Since $B\le C_d\Gamma_T$, $b_{B,T}\le C_d\ell_T$,
$\Gamma_T=O((\log\ell_T)^2)$, and
$1-\beta_{d,B}\le1-\alpha_d=1/p$,
\[
  B^2(\Gamma_Tb_{B,T})^{1-\beta_{d,B}}\le C_d \ell_T,
  \qquad
  (\Gamma_Tb_{B,T})^{1/(p+1)}B^{\upsilon_d}\le C_d \ell_T
\]
for $T\ge T_d$.  Using $1+(B-1)W\le B(1\vee W)$ proves that the hierarchical policy's expected regret is bounded by
\eqref{eq:joint-frontier-upper}.

\proofparagraph{The case $B\ge B_0$}
Suppose $B\ge B_0$. Define
\[
\sigma_T=
\left(\frac{A_d^{\mathrm{opt}}\ell_T}{T}\right)^{1/p},
\qquad
\Xi=
\left(\frac{T}{\ell_T}\right)^{d/p}
\wedge
\left(1+\frac{(B-1)W}{A_d^{\mathrm{opt}}\Gamma_T}\right).
\]
Since $B\ge B_0$ and $W\gtrsim_d\ell_T$,
\[
\Xi\gtrsim_d
\min\left\{
\left(\frac{T}{\ell_T}\right)^{d/p},\ell_T
\right\}\longrightarrow\infty
\]
uniformly over this regime.
Choose
\begin{equation}\label{eq:serialized-parent-choice}
\begin{aligned}
\widetilde s
&=
\max\left\{
\sigma_T,
\left(1+\frac{(B-1)W}
{A_d^{\mathrm{opt}}\Gamma_T}\right)^{-1/d}
\right\},\\
\widetilde r
&=\left(\frac{\widetilde s\ell_T}{T}\right)^{1/(p+1)},
\qquad
s=\lceil\widetilde s\rceil_2,\qquad
r=\lceil\widetilde r\rceil_2.
\end{aligned}
\end{equation}
Because $A_d^{\mathrm{opt}}\ge1$, this choice satisfies
$\widetilde s=\max\{\sigma_T,\Xi^{-1/d}\}$.
For sufficiently large $T_d$, the rounded scales lie in
$(0,1/16]$ and satisfy $r\le s$.

The memory condition of Lemma~\ref{lem:two-scale-upper} holds by
our choice of the coefficient in $W\ge C_d\ell_T$.
Since $s^{-d}\le\Xi$ and
$(1+\log\ell_T)^2\asymp\Gamma_T$, its batch requirement follows from
\begin{align*}
 C_d(1+s^{-d}/W)(1+\log\ell_T)^2
 &\le C_d'\Gamma_T+\frac{C_d'\Xi\Gamma_T}{W}\\
 &\le C_d'\Gamma_T+\frac{C_d'\Gamma_T}{W}
       +\frac{C_d'(B-1)}{A_d^{\mathrm{opt}}}
 \le B-1.
\end{align*}
The last inequality uses the fixed choices of
$A_d^{\mathrm{opt}}$ and $A_d^{\mathrm{switch}}$.

The choice $s\ge\sigma_T$ and
$r^{p+1}\asymp s\ell_T/T$ give
\[
  r^{-p}\ell_T\le C_d T \bigl(A_d^{\mathrm{opt}}\bigr)^{-1/(p+1)}.
\]
Thus the horizon condition in Lemma~\ref{lem:two-scale-upper}
holds by the choice of $A_d^{\mathrm{opt}}$, and the lemma applies.
Moreover,
\begin{equation}
  sr^{-p}\ell_T\asymp Tr.
\end{equation}
For the resulting policy $\mathcal A$, substituting
\eqref{eq:serialized-parent-choice} into $Tr$ gives
\begin{equation}\label{eq:serialized-envelope}
  \sup_{f\in\Lip_1(\X),\,\nu\in\Dclass(f)}\E_{\nu,\mathcal A}\bigl[\Reg_T(f)\bigr]
  \le C_d T^{\frac{d+2}{d+3}}\ell_T^{\frac1{d+3}}
  \Xi^{-\upsilon_d}.
\end{equation}
The two terms defining $\Xi$ correspond to statistical saturation
and the memory--batch budget. Since $A_d^{\mathrm{opt}}$ is fixed,
\begin{align*}
T^{\frac{d+2}{d+3}}\ell_T^{\frac1{d+3}}\Xi^{-\upsilon_d}
&\le C_d\left[
T^{\alpha_d}\ell_T^{1/(d+2)}
\vee T^{\frac{d+2}{d+3}}(1+(B-1)W)^{-\upsilon_d}
       \ell_T^{1/(d+3)}\Gamma_T^{\upsilon_d}\right]\\
&\le C_d\ell_T\left[
T^{\alpha_d}\vee
T^{\frac{d+2}{d+3}}(1+(B-1)W)^{-\upsilon_d}\right],
\end{align*}
where the last step uses $\Gamma_T=O((\log\ell_T)^2)$.
This proves \eqref{eq:joint-frontier-upper} in the large-batch
regime. The three constructions have batch boundaries fixed in
advance and together cover every $1\le B\le T$.
\end{proof}

\subsection{Dominant terms and resource boundaries}
\label{app:regime-comparisons}

We compare the three displayed terms in
Theorem~\ref{thm:joint-frontier}, using $Q_{T,B}$ from
\eqref{eq:normalized-frontier-terms}.
The algebraic comparisons hold for all admissible budgets.
They characterize minimax regret within the theorem's
logarithmic factor only when $W\ge C_d\ell_T$ and
$T$ is sufficiently large.

\paragraph{Dominance and memory saturation.}
The conditions stated in the main text follow by comparing
the normalized terms in \eqref{eq:normalized-frontier-terms}.
Equality between two terms gives a dominance boundary only
when their common value is at least the third term.

For example, the sequential and joint terms are equal at
$1+(B-1)W=T^{d/(d+2)}$, but this is a boundary between
dominant terms only when $Q_{T,B}\le1$.
Similarly, the batch-depth and joint terms are equal at
\begin{equation}\label{eq:batch-joint-crossing}
1+(B-1)W
=
T^{d/(d+2)}Q_{T,B}^{-d(d+3)}
=
B^{2d(d+3)}
T^{d(d+2)-d(d+3)\beta_{d,B}},
\end{equation}
and this crossing is visible above the sequential term
only when $Q_{T,B}\ge1$.
The sequential and batch-depth terms coincide when
$Q_{T,B}=1$ and are maximal there exactly when
$1+(B-1)W\ge T^{d/(d+2)}$.
All three terms coincide when
$Q_{T,B}=1$ and
$1+(B-1)W=T^{d/(d+2)}$.
An exact intersection need not occur at admissible integer
budgets.

If the memory threshold in Section~\ref{subsec:formal-frontier}
lies below $C_d\ell_T$, the joint term is never largest
in the matching-memory regime.

\paragraph{Increasing the batch budget.}
Direct subtraction of exponents gives
\[
\log Q_{T,B}
=
\frac{\alpha_d\log T}{(d+2)^B-1}-2\log B.
\]
For this comparison, extend $B$ to real values.
Both terms on the right decrease strictly with $B$.
For $T>1$, the sequential and batch-depth terms therefore
have a unique crossing in $(1,T)$, specified by
\begin{equation}\label{eq:statistical-batch-crossing}
\frac{\alpha_d\log T}{(d+2)^B-1}=2\log B.
\end{equation}
For fixed $d$, its location satisfies
$B\sim\log\log T/\log(d+2)$.
Indeed, at
$B=(1-\epsilon)\log\log T/\log(d+2)$,
the first term in $\log Q_{T,B}$ has order
$(\log T)^\epsilon$ and dominates $2\log B$;
at the corresponding $(1+\epsilon)$ scale, it tends to zero.
Integer budgets use the adjacent integers around this crossing.

For fixed $W>0$, the joint term also decreases with $B$.
Moreover, the ratio of the batch-depth term to the joint term
decreases strictly for real $B=b\ge2$.
Its logarithmic derivative is
\[
\beta'_{d,b}\log T-\frac2b
+\frac{W}{d(d+3)(1+(b-1)W)}
\le
\beta'_{d,b}\log T-\frac2b
+\frac{1}{d(d+3)(b-1)}
<0,
\]
because $\beta'_{d,b}<0$ and $d(d+3)\ge4$.
At $B=1$, the batch-depth term equals $T$ and is largest.
At $B=T$ and integer $W\ge1$, the sequential term is largest.
Thus, ignoring ties, increasing an integer batch budget can
visit batch-depth, joint, and sequential dominance in that
order, or skip the joint regime.
It cannot return from joint to batch-depth dominance.
The sequential term dominates precisely when both
$Q_{T,B}\le1$ and
$1+(B-1)W\ge T^{d/(d+2)}$.

These exact crossings concern the displayed expressions.
Terms differing only by logarithmic factors cannot be
distinguished at $\widetilde\Theta_d$ precision.
In particular, the first budget at which the sequential term
is largest need not equal the minimum budget for the
near-optimal regret target in
Corollary~\ref{cor:batch-complexity}.

\paragraph{Necessity of the three regret terms.}
Each term can exceed the other two by a polynomial factor
within the matching-memory regime.
For sufficiently large $T$, consider the following integer
budgets.

\emph{Sequential dominance.}
Take $B=T$ and $W=\lceil C_d\ell_T\rceil$.
Relative to $T^{\alpha_d}$, the batch-depth and joint terms
have orders
\[
T^{-2}
\qquad\text{and}\qquad
T^{-\frac{2}{d(d+2)(d+3)}}
\ell_T^{-\frac1{d(d+3)}},
\]
respectively.
Both are polynomially smaller than the sequential term.

\emph{Batch-depth dominance.}
Take $B=2$ and $W=\lceil T^{d/(d+2)}\rceil$.
The joint term has order $T^{\alpha_d}$, while
$\beta_{d,2}=(d+2)/(d+3)$ gives the batch-depth term
$T^{(d+2)/(d+3)}/4$.
Its ratio to either of the other terms has order
$T^{1/((d+2)(d+3))}$.
More generally, \eqref{eq:batch-joint-crossing} at $B=2$
reduces to $1+W=4^{d(d+3)}$.
Thus, throughout $W\ge C_d\ell_T$ and for sufficiently
large $T$,
$\mathfrak R_T(2,W)=
\widetilde\Theta_d(T^{(d+2)/(d+3)})$.

\emph{Joint dominance.}
Take $B=3$ and $W=\lceil C_d\ell_T\rceil$.
The joint term has order
\[
T^{\frac{d+2}{d+3}}
\ell_T^{-\frac1{d(d+3)}}.
\]
Since
$\alpha_d<\beta_{d,3}<(d+2)/(d+3)$,
it exceeds both the sequential and batch-depth terms by
polynomial factors.
For these budgets, the upper-bound proof uses the hierarchy.
Joint dominance therefore does not require using the
construction in groups.

These examples and the matching upper bound show that all
three terms in \eqref{eq:joint-frontier-lower} are needed
to characterize minimax regret uniformly up to logarithmic
factors.

\subsection{Proof of the batch-complexity corollary}
\label{app:batch-complexity-proof}

\begin{proof}[Proof of Corollary~\ref{cor:batch-complexity}]
\begin{equation}\label{eq:appendix-batch-complexity-bounds}
 c_d\left[\log\log T\vee
 \frac{T^{d/(d+2)}}{W\ell_T^{d(d+3)}}\right]
 \le\mathcal B_T(W;\varrho_d)
 \le C_d\left[\log\log T\vee\frac{T^{d/(d+2)}}W\right].
\end{equation}

Choose $\varrho_d$ larger than the constant in
\eqref{eq:joint-frontier-upper}, and set
\[
B_+:=
\left\lceil
C_d\left[
\log\log T\vee\frac{T^{d/(d+2)}}{W}
\right]
\right\rceil.
\]
Since $W\gtrsim_d\ell_T$, we have $B_+\le T$ for sufficiently
large $T$.
After increasing $C_d$,
\[
(B_+-1)W\ge c_dT^{d/(d+2)},
\qquad
T^{\beta_{d,B_+}-\alpha_d}
=\exp\!\left(
\frac{\alpha_d\log T}{(d+2)^{B_+}-1}
\right)
\le C_d.
\]
Substitution into \eqref{eq:joint-frontier-upper} gives
$\mathfrak R_T(B_+,W)\le\varrho_d\ell_TT^{\alpha_d}$ and therefore the upper
bound in \eqref{eq:appendix-batch-complexity-bounds}.

Conversely, the joint term in \eqref{eq:joint-frontier-lower} gives,
for every $\Lambda>0$,
\[
\mathfrak R_T(B,W)\le\Lambda T^{\alpha_d}
\quad\Longrightarrow\quad
1+(B-1)W\ge
c_dT^{d/(d+2)}\Lambda^{-d(d+3)}.
\]
If $\mathfrak R_T(B,W)\le\varrho_d\ell_TT^{\alpha_d}$, take
$\Lambda=\varrho_d\ell_T$. For sufficiently large $T$, subtracting
one and dividing by $W$ yields
\[
B\ge c_d\frac{T^{d/(d+2)}}{W\ell_T^{d(d+3)}}.
\]
The batch-depth term also requires $B^{-2}e^{u_B}\le C_d\ell_T$,
where
$u_B:=\alpha_d\log T/[(d+2)^B-1]$.
If $B\le c_d\log\log T$, with $c_d$ sufficiently small, then
$(d+2)^B\le(\log T)^{1/2}$, so
$u_B\ge c_d\sqrt{\log T}$ and the requirement fails.
Hence $B\ge c_d\log\log T$.
Combining the two necessary conditions proves
\eqref{eq:appendix-batch-complexity-bounds};
\eqref{eq:batch-complexity-tilde} follows after suppressing
powers of $\ell_T$.
\end{proof}
\subsection{Finite-arm comparison}
\label{app:finite-arm-calculation}

We derive the comparison of regret upper bounds in
Section~\ref{app:finite-arm-comparison}.
Recall $\ell_T=\log(eT)$ and the logarithmic batch overhead
$\Gamma_T$ from \eqref{eq:serialization-overhead}.

\paragraph{Admissible grids.}
For a regular $K$-point grid, Lipschitz continuity bounds
the approximation contribution to regret by
$C_dTK^{-1/d}$.
The formal version of \citet[Theorem~2]{huang2026few}
uses a block size $S\in[K]$.
For constants $c_0,C_0>0$ and $T\ge c_0K$, its memory,
batch count, and finite-arm regret are bounded by
\[
C_0S\ell_T,\qquad
C_0\frac KS\Gamma_T,\qquad
C_0\sqrt{KT\ell_T}\,\Gamma_T,
\]
respectively, after increasing $C_0$ if necessary.

Choose
$S=\min\{K,\lfloor W/(C_0\ell_T)\rfloor\}$.
For $W\ge2C_0\ell_T$, this gives
\[
\frac KS
\le
\max\left\{1,\frac{2C_0K\ell_T}{W}\right\}.
\]
Consequently, for sufficiently large $C$ and sufficiently
small $c>0$, the conditions
\[
W\ge C\ell_T,\qquad B\ge C\Gamma_T,\qquad
2\le K\le K_{\max}
:=
\left\lfloor\frac{cBW}{\ell_T\Gamma_T}\right\rfloor
\wedge
\left\lfloor\frac{T}{c_0}\right\rfloor
\]
ensure that this application fits the budgets.
The resulting continuum regret bound is
\[
C_dTK^{-1/d}
+
C_0\sqrt{KT\ell_T}\,\Gamma_T.
\]
Rounding the grid size to a $d$th power changes only
dimension-dependent constants.

\paragraph{Optimizing the direct guarantee.}
Suppress logarithmic factors for the following comparison
of polynomial orders.
Balancing $TK^{-1/d}$ and $\sqrt{KT}$ gives
$K\asymp_d T^{d/(d+2)}$.
If $K_{\max}\le T^{d/(d+2)}$, then
$\sqrt{KT}\le TK^{-1/d}$ throughout the permitted range,
so choosing $K$ of order $K_{\max}$ is optimal up to constants.
If $K_{\max}\ge T^{d/(d+2)}$, the balancing choice is available.
Thus
\[
\min_{2\le K\le K_{\max}}
\left\{TK^{-1/d}+\sqrt{KT}\right\}
\asymp_d
T^{\alpha_d}\vee TK_{\max}^{-1/d}.
\]
Restoring logarithmic factors, the optimized stated bound
has polynomial order
$T^{\alpha_d}\vee T(BW)^{-1/d}$.

Comparing this expression with the joint term gives the
regret penalties stated in Section~\ref{app:finite-arm-comparison}.

\paragraph{The one-dimensional example.}
For $d=1$, $W\asymp\log T$ with a sufficiently large constant,
and $B\asymp T^{1/4}/\log T$, the preceding applicability
conditions hold and $K_{\max}=\widetilde\Theta(T^{1/4})$.
At the optimizing grid size $K\asymp K_{\max}$,
\[
TK^{-1}=\widetilde\Theta(T^{3/4}),\qquad
\sqrt{KT}=\widetilde\Theta(T^{5/8}),\qquad
T^{3/4}(1+(B-1)W)^{-1/4}\asymp T^{11/16}.
\]
The direct-grid bound is therefore of polynomial order $T^{3/4}$.
Since $B\gg\log\log T$, our batch-depth term is bounded by the
statistical term of order $T^{2/3}$, which is smaller than $T^{11/16}$.
Theorem~\ref{thm:joint-frontier} thus gives
$\widetilde O(T^{11/16})$, a separation of polynomial order $T^{1/16}$
between the stated guarantees.
\section{Proof of the one-bit regret bounds}
\label{app:one-bit-timing}

We prove Theorem~\ref{thm:one-bit-timing} using the policy maps,
boundary states, and padding conventions of
Appendix~\ref{app:boundary-memory}.

\paragraph{Fixed boundaries.}\label{app:one-bit-static}
Conditional on the public seed, fixed boundaries restrict
each round's action to an instance-independent set of at
most $2^W$ points.
Lemma~\ref{lem:finite-action-codebook} therefore gives
\begin{equation}\label{eq:one-bit-fixed-boundaries}
\mathfrak R_T^{\mathrm{fix}}(B,W)\ge c_dT2^{-W/d}.
\end{equation}
For $W=1$, combining this bound with $\Reg_T(f)\le T$
gives $\mathfrak R_T^{\mathrm{fix}}(B,1)=\Theta_d(T)$;
compare \citet[Section~4]{zhu2025lipschitz}.
It remains to prove the matching bounds with adaptive boundaries.

\subsection{Lower bound with adaptive boundaries}
\label{app:one-bit-lower}

\paragraph{Two-state contraction.}\label{app:one-bit-extraction}
The entropy bound \eqref{eq:boundary-information} controls the
number of bits carried across boundaries.
For one bit, we additionally bound how well the final state of
each batch distinguishes disjoint local alternatives.
The following estimate is uniform in the batch length and can
then be composed across nonterminal boundaries.

Use $H^2(P,Q)=\int(\sqrt{dP}-\sqrt{dQ})^2$, with respect to any
common dominating measure, and $\operatorname{TV}(P,Q)\le H(P,Q)$.
For deterministic maps $\kappa_t:\{0,1\}^2\to\{0,1\}$, put
\[
K_t^{(u)}(i,j)=(\tfrac12-u)\mathbf1\{\kappa_t(i,0)=j\}
 +(\tfrac12+u)\mathbf1\{\kappa_t(i,1)=j\}.
\]
Inputs are independent under each law. The reference has
$u_{0,t}=0$, and alternatives $v\in[K]$ have biases satisfying
\begin{equation}\label{eq:one-bit-disjoint-inputs}
|u_{v,t}|\le\delta\le\tfrac14,\qquad
\#\{v:u_{v,t}\ne0\}\le1\quad\text{for every }t.
\end{equation}
Starting from a common state $i_\circ$, let
$Q_{v,0}=\delta_{i_\circ}$ and $Q_{v,t}=Q_{v,t-1}K_t^{(u_{v,t})}$.

\begin{lemma}[Two-state spatial separation]\label{lem:one-bit-extraction}
Uniformly in $n,K$ and the transitions,
\begin{equation}\label{eq:one-bit-extraction}
\sum_{v=1}^K H^2(Q_{0,n},Q_{v,n})\le48\delta^2.
\end{equation}
\end{lemma}
\begin{proof}
Let $\mathcal I=\{t:\exists i,\ \kappa_t(i,0)\ne\kappa_t(i,1)\}$.
Outside $\mathcal I$, all laws use the same kernel.
For $t\in\mathcal I$, the reference rows have output-one probabilities
in $\{0,1/2,1\}$, at least one equal to $1/2$.
Their common mass $a_t=\sum_j\min_iK_t^{(0)}(i,j)$ is therefore
at least $1/2$. The kernels
$J_t(i,j)=\min_iK_t^{(0)}(i,j)/a_t$ and $L_t=2K_t^{(0)}-J_t$
are Markov, and $J_t$ has identical rows. Joint convexity and
data processing give
\[
\begin{aligned}
H^2(PK_t^{(0)},QK_t^{(0)})&\le\tfrac12H^2(P,Q),\\
H^2(QK_t^{(0)},QK_t^{(u)})
&\le H^2(\operatorname{Ber}(1/2),\operatorname{Ber}(1/2+u))
\le4u^2.
\end{aligned}
\]
Apply the triangle inequality and
$(a+b)^2\le\tfrac32a^2+3b^2$.
For $D_t=\sum_vH^2(Q_{0,t},Q_{v,t})$, disjointness gives
\[
D_t\le
\begin{cases}
D_{t-1},&t\notin\mathcal I,\\
\tfrac34D_{t-1}+12\delta^2,&t\in\mathcal I.
\end{cases}
\]
Since $D_0=0$, summing the geometric series proves the bound.
\end{proof}

\paragraph{Composition across boundaries.}\label{app:one-bit-chain}
For laws $P,Q$ of a finite sequential record $\mathsf Z_{1:k}$
with common support and conditional kernels $P_b,Q_b$,
\begin{equation}\label{eq:one-bit-chain}
H^2(P,Q)\le4\E_P\sum_{b=1}^k
H^2\bigl(P_b(\cdot\mid\mathsf Z_{<b}),
         Q_b(\cdot\mid\mathsf Z_{<b})\bigr).
\end{equation}
To prove \eqref{eq:one-bit-chain} with expectation under $P$, set
$L_b=dQ_{1:b}/dP_{1:b}$, $L_0=1$, $\xi_b=L_b/L_{b-1}$, and
$\vartheta(x)=(x-1)^2/(x+1)$.
For $x>0$ and $u\ge0$, the identities
\[
\vartheta(xu)-\vartheta(x)-x\vartheta'(x)(u-1)
=\frac{4x^2(u-1)^2}{(1+x)^2(1+xu)},\qquad
(\sqrt u+1)^2\le(1+x^{-1})(1+xu)
\]
imply, since $\E_P[\xi_b\mid\mathsf Z_{<b}]=1$,
$\E_P[\vartheta(L_b)-\vartheta(L_{b-1})\mid\mathsf Z_{<b}]
\le4\E_P[(\sqrt{\xi_b}-1)^2\mid\mathsf Z_{<b}]$.
Sum over $b$ and use $(\sqrt x-1)^2\le\vartheta(x)$.
For a common independent seed law $\mu$, the same inequality
holds conditionally and integrates by
$H^2(\mu(dw)P^w,\mu(dw)Q^w)=\int H^2(P^w,Q^w)\,\mu(dw)$.

\begin{proof}[Proof of the lower bound in Theorem~\ref{thm:one-bit-timing}]
For $0<\delta\le1/32$, choose
$\{z_v:v\in[K]\}=\{2\delta+4\delta j:
0\le j<\lfloor(4\delta)^{-1}\rfloor\}^d$.
Then $K\ge(8\delta)^{-d}$ and the radius-$\delta$ balls are
disjoint. Using the notation of Appendix~\ref{app:lower}, set
\begin{equation}\label{eq:one-bit-tent-family}
f_0(x)=\tfrac12,\qquad
f_v(x)=\tfrac12+(\delta-\|x-z_v\|_\infty)_+,\qquad
G_v=B_\infty(z_v,\delta/2).
\end{equation}
These are admissible Bernoulli means, with
$\Delta_{f_v}\ge(\delta/2)\mathbf1_{G_v^\complement}$.

Fix $\mathcal A\in\mathfrak A_{B,1}$. At a common boundary history
$\mathcal H_{b-1}=\sigma(\omega,\bar M_1,\ldots,\bar M_{b-1})$,
Lemma~\ref{lem:boundary-factorization} fixes the next batch's
initial bit, batch boundary, and actions.
The update maps, restricted to binary rewards, therefore satisfy
Lemma~\ref{lem:one-bit-extraction}, with disjoint biases from
\eqref{eq:one-bit-tent-family}.
Let $Q_{v,b}$ be the conditional law of $\bar M_b$ given this
history under $f_v$. If the reconstructed batch boundary is $T$ or
the horizon was already reached, padding makes all these laws
the same point mass; otherwise their summed squared Hellinger
distance is at most $48\delta^2$.
Interior Bernoulli means give common support conditional on each
seed. Applying \eqref{eq:one-bit-chain} and then integrating
against the common seed law, the laws
$P_v=\mathcal L_v(\omega,\mathbf M)$ satisfy
\begin{equation}\label{eq:one-bit-record-evidence}
\sum_vH^2(P_0,P_v)
\le4\E_0\sum_{b=1}^{B-1}\sum_vH^2(Q_{0,b},Q_{v,b})
\le192(B-1)\delta^2.
\end{equation}
The same reconstruction makes $F_v=N_{G_v}(T)/T$ a common
function of this record, and disjointness gives $\sum_vF_v\le1$.
Thus
\begin{equation}\label{eq:one-bit-occupancy}
\frac1K\sum_v\E_vF_v
\le\frac1K+\frac1K\sum_v\operatorname{TV}(P_0,P_v)
\le\frac1K+\sqrt{\frac{192(B-1)\delta^2}{K}}.
\end{equation}
Choose $\delta=c_dB^{-1/(d+2)}$ with $c_d\le1/32$ and
$192\,8^dc_d^{d+2}\le1/16$. The right side is at most $1/2$, so
$\sup_v\E_v\Reg_T(f_v)\ge(\delta T/2)
(1-K^{-1}\sum_v\E_vF_v)\ge\delta T/4$.
Taking the infimum proves the bound, including $B=1$, when the
boundary-state tuple is empty.
\end{proof}

\subsection{Upper bound with adaptive boundaries}
\label{app:one-bit-policy}

Fix $\rho=2^{-E}\le1/2$, $\rho_e=2^{-e}$, and
$q_\rho=\lceil A_d^{\mathrm{conf}}\log(e/\rho)\rceil$.
The construction has $E$ phases.
Write $t_e$ for the end of phase $e$, $U_e$ for a deterministic
upper bound on $t_e$, and $Q_E$ for the total number of exploration batches.

\begin{lemma}[One-bit fixed-accuracy implementation]
\label{lem:one-bit-fixed-accuracy}
For sufficiently large $A_d^{\mathrm{conf}}$, the construction below has
exactly two mutable configurations and, for every path,
\begin{equation}\label{eq:one-bit-resource-bound}
U_e\le C_dq_\rho\rho_e^{-d-2}\quad(e\in[E]),\qquad
Q_E\le t_E\le U_E.
\end{equation}
If $U_E\le T/2$, appending exploitation gives a policy $\mathcal A$ satisfying
\begin{equation}\label{eq:one-bit-regret-bound}
\sup_{f\in\Lip_1(\X),\,\nu\in\Dclass(f)}
\E_{\nu,\mathcal A}\bigl[\Reg_T(f)\bigr]
\le C_d\bigl(q_\rho\rho^{-d-1}+T\rho\bigr).
\end{equation}
The policy belongs to $\mathfrak A_{B,1}$ whenever also $Q_E+1\le B$.
\end{lemma}

\paragraph{Comparisons on successive grids.}
Use the $K_e=2^{de}$ midpoint representatives of
Section~\ref{subsec:fixed-geometry}, with labels
$\mathcal V_e=\{J_{e-1}+1,\ldots,J_e\}$, where
$J_0=0$ and $J_e=\sum_{a\le e}K_a\le2K_e$; set $x_0=x_\circ$.
Then $\max_{v\in\mathcal V_e}f(x_v)\ge f^\star-\rho_e$.
Prescribe
\begin{equation}\label{eq:one-bit-stage-schedule}
n_\ell=q_\rho2^\ell,\qquad
a_\ell=\sqrt{q_\rho/n_\ell},\qquad
L_e=6+2e.
\end{equation}
Thus $a_{L_e}=\rho_e/8$ and $n_{L_e}=64q_\rho\rho_e^{-2}$.
Traverse stages $(e,v,\ell)$ lexicographically, with
$\ell=0,\ldots,L_e$ for each challenger $v$.
The comparison variables are an incumbent label $I$, a flag
$F\in\{0,1\}$ indicating whether the current comparison has
finished, and the sum $C$ of the query bits in the current stage.
We first specify their updates, then realize them through batch
boundary times using a single mutable bit.

Each stage makes $2n_\ell$ queries.
In query $j$, pull $x_v$ if $F=0$ and $j\le n_\ell$,
and $x_I$ otherwise. Set $\sigma_j=\mathbf1\{j>n_\ell\}$.
If the query batch starts at time $t$, observe $Y_{t+1}$ and set
$M_{t+1}=\sigma_j+(1-2\sigma_j)\mathbf1\{U\le Y_{t+1}\}$,
where $U$ is the unused uniform seed coordinate assigned to
this query.
Set $\widehat D_\ell=(C-n_\ell)/n_\ell$; when $F=0$,
this estimates $f(x_v)-f(x_I)$.
At stage end, apply
\begin{equation}\label{eq:one-bit-comparison-rule}
(I^+,F^+)=
\begin{cases}
(v,1),&F=0,\ \widehat D_\ell>2a_\ell,\\
(I,1),&F=0,\ \widehat D_\ell<-2a_\ell,\\
(I,F),&\text{otherwise},
\end{cases}
\qquad I'=I^+,\quad F'=\mathbf1\{\ell<L_e\}F^+.
\end{equation}
The incumbent and flag remain unchanged within a stage.
The factor $\mathbf1\{\ell<L_e\}$ resets the flag for the next
challenger after the last stage.
Only the current challenger can replace the incumbent, so the
ordered traversal makes $I$ nondecreasing on every path.

\paragraph{Recovering comparison variables from time.}
Each stage has the fixed batch list
\[
(\mathsf Q_1,\mathsf T_1,\ldots,
\mathsf Q_{2n_\ell},\mathsf T_{2n_\ell},\mathsf C):
\]
one query and transfer per observation, then one control batch.
An alignment batch precedes every phase.
Finished comparisons keep their slots, use incumbent actions,
and ignore further decisions.
Thus $Q_E$ is deterministic and includes query, transfer,
control, and alignment batches; exploitation has batch index $Q_E$.

Put $R_e=2n_{L_e}+1$ and
\[
\operatorname{enc}_e(I,F,C)=2R_eI+R_eF+C,\qquad
\operatorname{dec}_e(z)=\left(
\left\lfloor\frac z{2R_e}\right\rfloor,
\left\lfloor\frac{z\bmod2R_e}{R_e}\right\rfloor,z\bmod R_e\right).
\]
At query starts the invariant is
\begin{equation}\label{eq:one-bit-clock-code}
t=D+z,\qquad z=\operatorname{enc}_e(I,F,C),\qquad
0\le I\le J_e,\quad0\le C<R_e,\quad M_t=0.
\end{equation}
The public batch index determines $D$ and all stage indices.
The current round $t$ then determines
$(I,F,C)=\operatorname{dec}_e(t-D)$.
Only $M_t$ is the mutable memory state.

\paragraph{Transferring a bit and updating the comparison.}
A query starting at time $t$ leaves the state $M_{t+1}$ defined above.
At time $t+1$, the transfer batch fixes its length at $1+M_{t+1}$,
pulls $x_I$, and sets $M_u=0$ after each observation:
\begin{equation}\label{eq:one-bit-clock-transfer}
\begin{aligned}
(t,D,z,0)
&\xrightarrow{\mathsf Q}(t+1,D+1,z,M_{t+1})\\
&\xrightarrow{\mathsf T}
(t+2+M_{t+1},D+2,z+M_{t+1},0).
\end{aligned}
\end{equation}
At the transfer start, $z=(t+1)-(D+1)$ determines the incumbent
through $\operatorname{dec}_e$.
Both batch lengths are positive integers fixed at their starts.
After $j$ pairs, $C\le j\le2n_\ell<R_e$, so no overflow
changes the incumbent or flag.

At stage end, compute $(I',F')$ from
\eqref{eq:one-bit-comparison-rule}.
For $z'=\operatorname{enc}_e(I',F',0)$, put
$g_{e,\ell}=2n_\ell+1$ if $\ell<L_e$, and $g_{e,L_e}=2R_e$.
The control batch pulls $x_{I'}$, sets $M_u=0$ after each
observation, and has length
\begin{equation}\label{eq:one-bit-control-length}
\lambda=g_{e,\ell}+z'-z,\qquad
 t'=t+\lambda=(D+g_{e,\ell})+z'.
\end{equation}
Since $I'\ge I$, for $\ell<L_e$ one has $F'\ge F$ and
$\lambda\ge2n_\ell+1-C\ge1$; at $\ell=L_e$,
$\lambda\ge2R_e-R_eF-C\ge1$.
Thus every control resets $C$, and the last level resets $F$
for the next challenger.

\paragraph{Phase alignment.}
Initialize $U_0=t_0=I_0=0$ and prescribe
\[
P_e=K_e\sum_{\ell=0}^{L_e}(4n_\ell+g_{e,\ell}),\qquad
 a_e^{\mathrm{clk}}=U_{e-1}+1,\qquad
 U_e=a_e^{\mathrm{clk}}+P_e+2R_eJ_e.
\]
At the end of phase $e-1$, recover $I_{e-1}$ from the
phase ending time and choose an
alignment batch on $x_{I_{e-1}}$ ending at
$t=a_e^{\mathrm{clk}}+2R_eI_{e-1}$, with $M_u=0$ after each observation.
Its length is $U_{e-1}-t_{e-1}+1+2R_eI_{e-1}\ge1$.
At this boundary, $D=a_e^{\mathrm{clk}}$ and $F=C=M_t=0$.
Each query--transfer pair advances $D$ by two,
and each control advances it by $g_{e,\ell}$, giving
\begin{equation}\label{eq:one-bit-phase-end}
t_e=a_e^{\mathrm{clk}}+P_e+2R_eI_e\le U_e,\qquad
I_e=\frac{t_e-a_e^{\mathrm{clk}}-P_e}{2R_e},\qquad F=C=M_{t_e}=0.
\end{equation}
The second identity recovers the incumbent label
from the phase ending time.

\begin{proof}[Proof of Lemma~\ref{lem:one-bit-fixed-accuracy}]
\label{app:one-bit-analysis}
\proofparagraph{Resources and policy maps}
For $n_\star=n_{L_e}$, geometric summation gives
$\sum_{\ell\le L_e}n_\ell<2n_\star$ and $L_e+1\le n_\star$.
Hence $P_e+2R_eJ_e\le C_dK_en_\star$
and $K_en_\star=64q_\rho2^{(d+2)e}$, so
$U_e=\sum_{a\le e}(1+P_a+2R_aJ_a)
\le C_dq_\rho\rho_e^{-d-2}$ for every prefix.
By construction, $t_e\le U_e$ for every $e$, and positive
batch lengths give $Q_E\le t_E$.
This proves \eqref{eq:one-bit-resource-bound}, including all
waiting rounds.
Query updates leave one bit and all other updates leave zero.

When $U_E\le T/2$ and $Q_E+1\le B$, the recursions define
$(\psi_b,\phi_{b,u})$ on the finite set of schedule-reachable
time/bit pairs, with $t<\psi_b\le U_E<T$ for $b<Q_E$.
At $b=Q_E$, decode $I_E$ using \eqref{eq:one-bit-phase-end}
and exploit to $T$. On all other inputs use batch boundary $T$ and
action $x_\circ$, with padding at $T$ as in
Appendix~\ref{app:boundary-memory}; initialize $M_0=0$.
Updates are the query map or zero, selected by the public batch
index, and use only the arguments of $\mathcal U_{b,u}$.
Thus the maps are total and Borel, $\psi_{B-1}=T$, and every
batch's actions and boundary are fixed at its start.

\proofparagraph{Concentration}
Fix $f\in\Lip_1(\X)$ and $\nu\in\Dclass(f)$.
At stage $\zeta=(e,v,\ell)$, let $\mathcal H_\zeta$ be its
starting history with only previously used seed coordinates
exposed, and let $(I_\zeta,F_\zeta)$ be its starting pair.
Query times are predictable in this filtration, and each query
uses a fresh uniform independent of its reward.
On $\{F_\zeta=0\}$ the two actions remain fixed throughout the
stage; the first $n_\ell$ written bits have conditional mean
$f(x_v)$, and the last $n_\ell$ have mean $1-f(x_{I_\zeta})$.
Write $\widehat D_\zeta$ for this stage's estimate $\widehat D_\ell$
and $D_\zeta=f(x_v)-f(x_{I_\zeta})$ for its mean difference.
Hoeffding's bound for these bounded increments gives
\[
\PP\bigl(F_\zeta=0,\ |\widehat D_\zeta-D_\zeta|>a_\ell
 \mid\mathcal H_\zeta\bigr)
\le2\mathbf1\{F_\zeta=0\}e^{-n_\ell a_\ell^2}
=2\mathbf1\{F_\zeta=0\}e^{-q_\rho}.
\]
In particular, this conditioning does not expose the entire seed.
For the event $\mathcal E$ that every unfinished stage is accurate,
\begin{equation}\label{eq:one-bit-confidence-event}
\PP(\mathcal E^\complement)
\le2e^{-q_\rho}\sum_{e=1}^E K_e(L_e+1)
\le C_d\rho^{-d}\log(e/\rho)e^{-q_\rho}\le\rho.
\end{equation}

\proofparagraph{Incumbents and regret}
On $\mathcal E$, acceptance strictly improves the incumbent mean,
and rejection discards a strictly worse challenger.
At an unresolved last level, $D_\zeta\le3a_{L_e}$.
Consequently,
\begin{equation}\label{eq:one-bit-incumbent-invariant}
f(x_{I_e})\ge\max_{v\in\mathcal V_e}f(x_v)-3a_{L_e}
\ge f^\star-11\rho_e/8,\qquad \Delta_f(x_{I_e})\le2\rho_e.
\end{equation}
Every incumbent used in phase $e$, including alignment and control,
has gap at most $4\rho_e$: this is trivial for $e=1$ and follows
from $2\rho_{e-1}=4\rho_e$ and improving means thereafter.

Let $N_{e,v}$ count unfinished challenger queries and
$\Delta_v=\Delta_f(x_v)$.
For $\Delta_v<8\rho_e$, the full schedule has at most
$\sum_\ell n_\ell\le C_dq_\rho\rho_e^{-2}$ challenger pulls.
Otherwise let $\ell_\star=\min\{\ell:a_\ell\le\Delta_v/8\}$.
It exists and is positive; an unfinished stage there satisfies
\[
\widehat D_{\ell_\star}
\le-\Delta_v+4\rho_e+a_{\ell_\star}
\le-3\Delta_v/8<-2a_{\ell_\star}.
\]
The challenger is therefore rejected by that level, and
$\sum_{\ell\le\ell_\star}n_\ell\le C_dq_\rho/\Delta_v^2$.
Together, the two cases give
\[
N_{e,v}\le\frac{C_dq_\rho}{(\Delta_v\vee\rho_e)^2},\qquad
\Delta_vN_{e,v}\le\frac{C_dq_\rho}{\rho_e}.
\]
All other pulls use an incumbent.
Using $t_e-t_{e-1}\le U_e$ and
\eqref{eq:one-bit-resource-bound}, including all waiting rounds,
\[
\sum_{t=t_{e-1}+1}^{t_e}\Delta_f(A_t)
\le\sum_{v\in\mathcal V_e}\Delta_vN_{e,v}
   +4\rho_e(t_e-t_{e-1})
\le C_dq_\rho\rho_e^{-d-1}.
\]
Sum geometrically, exploit $x_{I_E}$, and charge
$T\PP(\mathcal E^\complement)\le T\rho$ once to obtain
\eqref{eq:one-bit-regret-bound}.
\end{proof}

\begin{proof}[Proof of the upper bound in Theorem~\ref{thm:one-bit-timing}]
For $B=1$, use a fixed action. For $B\ge2$, put
$\widetilde\rho_B=(\log(eB)/B)^{1/(d+2)}$ and choose $A_d^{\mathrm{rad}}\ge1$
after the confidence and resource constants.
If $A_d^{\mathrm{rad}}\widetilde\rho_B>1/4$, a fixed action suffices since
$T\le4A_d^{\mathrm{rad}}T\widetilde\rho_B$.
Otherwise take $\rho=\lceil A_d^{\mathrm{rad}}\widetilde\rho_B\rceil_2\le1/2$
using Section~\ref{subsec:fixed-geometry}. Then
\[
q_\rho\le C_d\log(eB),\qquad
U_E\le\frac{C_d}{(A_d^{\mathrm{rad}})^{d+2}}B\le B/2,
\qquad q_\rho\rho^{-d-1}\le C_dT\rho.
\]
Since $Q_E\le t_E\le U_E$, one has $Q_E+1\le B$ for $B\ge2$.
Also $t_E\le U_E\le B/2\le T/2$, leaving a nonempty exploitation
batch. Lemma~\ref{lem:one-bit-fixed-accuracy} gives regret at most
$C_dT\rho\le C_d'T\widetilde\rho_B$.
\end{proof}

\end{document}